\documentclass[12pt]{article}

\usepackage{times}
\usepackage{footmisc}
\usepackage{comment}
\usepackage{booktabs}
\usepackage{dcolumn}
\usepackage{amsmath}

\usepackage{graphicx}
\usepackage[font=small]{caption}

\usepackage{fullpage}
\usepackage{url}
\usepackage{graphicx}
\usepackage{epsfig}
\usepackage{amssymb}
\usepackage{amsmath}
\usepackage{amsthm}
\usepackage{mathtools}
\usepackage{color}
\usepackage{multirow}
\usepackage{tabularx}
\usepackage{fancyhdr}

\newcommand{\etc}{\textit{etc.}}
\newcommand{\eg}{\textit{e.g. }}
\newcommand{\ie}{\textit{i.e. }}

\newcommand{\calA}{\mathcal{A}}
\newcommand{\calB}{\mathcal{B}}

\newcommand{\calJ}{\mathcal{J}}

\newcommand{\calL}{\mathcal{L}}

\newcommand{\calP}{\mathcal{P}}
\newcommand{\calS}{\mathcal{S}}
\newcommand{\AF}{\mathit{AF}}
\newcommand{\Ag}{\mathit{Ag}}

\newcommand{\comp}[1]{\mathit{Comp(#1)}}
\newcommand{\stab}[1]{\mathit{Stab(#1)}}
\newcommand{\semi}[1]{\mathit{Semi(#1)}}
\newcommand{\pref}[1]{\mathit{Pref(#1)}}
\newcommand{\grnd}[1]{\mathit{Grnd(#1)}}

\newcommand{\myin}{\mathtt{in}}
\newcommand{\myout}{\mathtt{out}}
\newcommand{\myundec}{\mathtt{undec}}

\newcommand{\calLAP}{\calL\calA\calP}

\newtheorem{remark}{Remark}

\newtheorem{theorem}{\bf Theorem}
\newtheorem{lemma}{\bf Lemma}
\newtheorem{proposition}{\bf Proposition}
\newtheorem{corollary}{\bf Corollary}

\newtheorem{definition}{\bf Definition}
\newtheorem{example}{\bf Example}
\newtheorem*{acknowledgment}{\bf Acknowledgment}

\usepackage{authblk}

\begin{document}

\title{Judgment Aggregation in Multi-Agent Argumentation}
\author[1]{Edmond Awad}
\author[2]{Richard Booth}
\author[3]{Fernando Tohm\'{e}}
\author[1,4,5,$\dagger$]{Iyad Rahwan}
\affil[1]{Masdar Institute of Science \& Technology, UAE}
\affil[2]{Mahasarakham University, Thailand}
\affil[3]{Universidad Nacional del Sur, Argentina}
\affil[4]{University of Edinburgh, UK}
\affil[5]{MIT, USA}

\affil[$\dagger$]{Correspondence should be addressed to
\texttt{irahwan@acm.org}}
\date{}
\maketitle


\begin{abstract}
Given a set of conflicting arguments, there can exist multiple plausible opinions about which arguments should be accepted, rejected, or deemed undecided. We study the problem of how multiple such judgments can be aggregated. We define the problem by adapting various classical social-choice-theoretic properties for the argumentation domain. We show that while argument-wise plurality voting satisfies many properties, it fails to guarantee the \emph{collective rationality} of the outcome. We then present more general results, proving multiple impossibility results on the existence of \emph{any} good aggregation operator. After characterising the sufficient and necessary conditions for satisfying collective rationality, we study whether restricting the domain of argument-wise plurality voting to classical semantics allows us to escape the impossibility result. We close by mentioning a couple of graph-theoretical restrictions under which the argument-wise plurality rule does produce collectively rational outcomes. In addition to identifying fundamental barriers to collective argument evaluation, our results contribute to research at the intersection of the argumentation and computational social choice fields.

\end{abstract}
\clearpage

\section{Introduction}
\label{section:intro}

Argumentation has recently become one of the key approaches to automated reasoning and rational interaction in Artificial Intelligence \cite{benchcapon:dunne:2007,rahwan:simari:2009}. A key milestone in the development of argumentation in AI has been Dung's landmark framework \cite{dung:1995}, known as abstract argumentation framework (AAF). Arguments are viewed as abstract entities (a set $\calA$), with a binary \emph{defeat} relation (denoted $\rightharpoonup$) over them. The defeat relation captures the fact that one argument somehow attacks or undermines another. This view of argumentation enables high-level analysis while abstracting away from the internal structure
of individual arguments. In Dung's approach, given a set of arguments and a defeat relation, a rule specifies which arguments should
be accepted. 

Often, there are multiple reasonable ways in which an agent may evaluate a given argument structure (\eg accepting only conflict-free, self-defending sets of arguments). Each possible evaluation corresponds to a so-called \emph{extension} \cite{dung:1995} or \emph{labelling} \cite{caminada:2006,caminada:gabbay:2009}. Different argumentation semantics yield different restrictions on the possible extensions. Most previous research has focused on evaluating and comparing different semantics based on the (objective) logical properties of their extensions \cite{baroni:giacomin:2007}.

One of the essential properties, which is common, is the condition of \emph{admissibility}: that accepted arguments must not attack one another, and must defend themselves against counter-arguments, by attacking them back. A stronger notion is called \emph{completeness}, and is captured, in terms of labelling, in the following two conditions:
\begin{enumerate}
	\item An argument is labelled \emph{accepted} (or $\myin$) if and only if all its defeaters are rejected (or $\myout$).
	\item An argument is labelled \emph{rejected} (or $\myout$) if and only if at least one of its defeaters is accepted (or $\myin$).
\end{enumerate}
Otherwise, an argument may be labelled $\myundec$. Thus, evaluating a set of arguments amounts to labelling each argument using a labelling function $L ~:~ \calA \rightarrow \{\myin, \myout, \myundec\}$ to capture these three possible labels. Any labelling that satisfies the above conditions is also called a \emph{legal labelling}. We will often use \emph{legal labelling} and \emph{complete labelling} interchangeably.

The above conditions attempt to evaluate arguments from a single point of view. Indeed, most research on formal models of argumentation discounts the fact that argumentation takes place among self-interested agents, who may have conflicting opinions and preferences over which arguments end up being accepted,
rejected, or undecided. Consider the following simple example.
\begin{example}[A Murder Case]\label{example:murder}
A murder case is under investigation. To start with, there is an argument that the suspect should be presumed innocent ($a_3$). However, there is evidence that he may have been at the crime scene at the time ($a_2$), which would counter the initial presumption of innocence. There is also, however, evidence that the suspect was attending a party that day ($a_1$). Clearly, $a_1$ and $a_2$ are mutually defeating arguments since the suspect can only be in one place at any given time. Hence, we have a set of arguments $\{a_1, a_2, a_3 \}$ and a defeat relation $\rightharpoonup = \{(a_1, a_2), (a_2, a_1), (a_2, a_3) \}$. There are three possible labellings that satisfy the above conditions:
    \begin{itemize}
      \item $L(a_1) = \myin$, $L(a_2) = \myout$, $L(a_3) = \myin$.
      \item $L'(a_1) = \myout$, $L'(a_2) = \myin$, $L'(a_3) = \myout$.
      \item $L''(a_1) = \myundec$, $L''(a_2) = \myundec$, $L''(a_3) = \myundec$.
    \end{itemize}
The graph and possible labellings are depicted in Figure \ref{fig:murder}.
\end{example}

\begin{figure}[ht]
    \centering
  \includegraphics[scale=1]{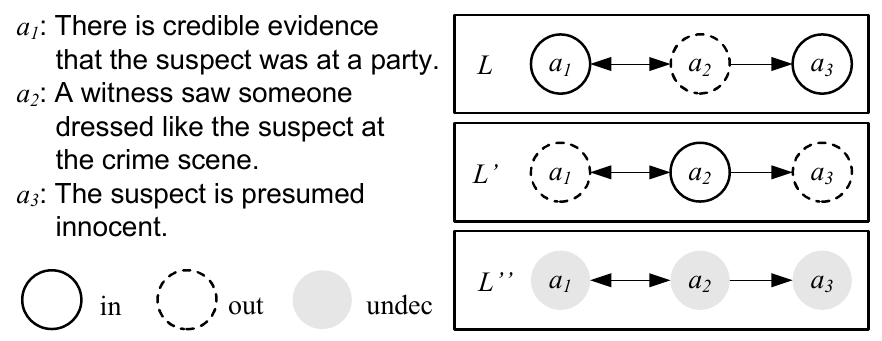}
  \caption{Argument graph with three possible labellings}\label{fig:murder}
\end{figure}

Example \ref{example:murder} highlights a situation in which multiple points of view can be taken, depending on whether one decides to accept the argument that the suspect was at the party or the crime scene. The question we explore in this paper can be highlighted through the following example, extending Example \ref{example:murder}.

\begin{example}[Three Detectives]\label{example:murder_aggregation}
A team of three detectives, named $1$, $2$, and $3$, have been assigned to the murder case described in Example \ref{example:murder}. Each detective's judgment can only correspond to a legal labelling (otherwise, her judgment can be discarded). Suppose that each detective's judgment is such that $L_1 = L$, $L_2 = L'$ and $L_3 = L'$. That is, detectives $2$ and $3$ agree but differ with detective $1$. These labellings are depicted in the labelled graph of Figure \ref{fig:murder_aggregation}. The detectives must decide which (aggregated) argument labelling best reflects their collective judgment.
\end{example}

\begin{figure}[ht]
  \centering
  \includegraphics[scale=1.0]{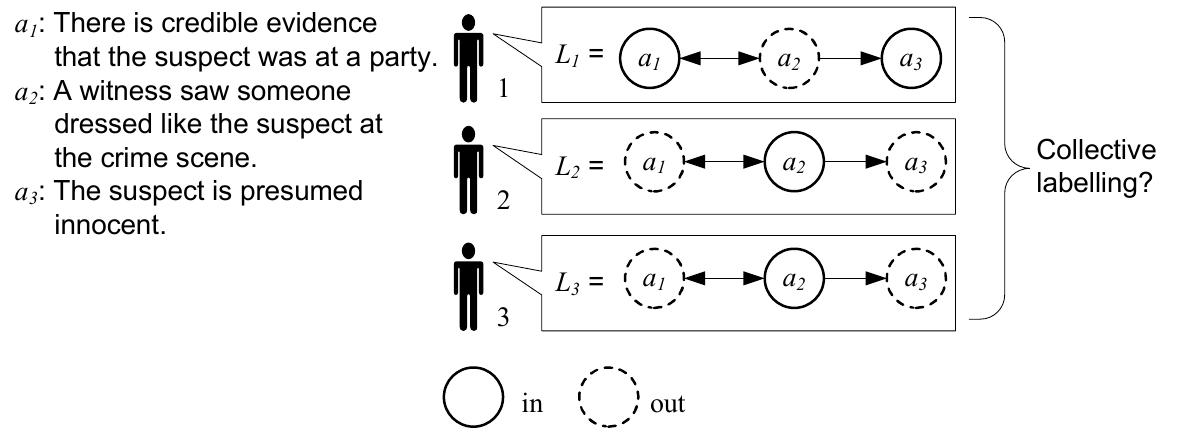}
  \caption{Three detectives with different judgments}\label{fig:murder_aggregation}
\end{figure}

Example \ref{example:murder_aggregation} highlights an aggregation problem, similar to the problem of preference aggregation \cite{arrow:etal:2002,gartner:2006,vincke1982aggregation} and the problem of judgment aggregation on propositional formulae \cite{list:puppe:2009,list2010theory,list2010introduction,grossi2014judgment}. It is perhaps obvious in this particular example that $a_3$ must be rejected (and thus the defendant be considered guilty), since most detectives seem to think so. For the same reason, $a_1$ must be rejected and $a_2$ must be accepted. Thus, labelling $L'$ (see Example \ref{example:murder}) wins by majority. As we shall see in our analysis below, things are not that simple, and counter-intuitive situations may arise. We summarise the main question asked in the paper as follows.

\begin{quote}
\emph{Given a set of agents, each with a specific subjective evaluation (\ie labelling) of a given set of conflicting arguments, how can agents reach a collective decision on how to evaluate those arguments?}
\end{quote}

While Arrow's Impossibility Theorem can be expected to ensue for this problem~\cite{arrow:1951},\footnote{Arrow's Theorem claims that four quite natural constraints, that capture abstractly the properties of a democratic aggregation process, cannot be simultaneously satisfied.} there exist many differences between labellings and preference relations (for which Arrow's result apply), stemming from their corresponding order-theoretic characterisations. In other words, aggregating preferences assumes that agents submit a full order of preferences over candidates, while in labelling aggregation, agents submit their top labelling for a set of logically connected arguments. 

The problem of labelling aggregation is more comparable to the judgment aggregation problem \cite{list:puppe:2009,list2010theory,list2010introduction,grossi2014judgment}, by considering arguments as propositions which are logically connected by the conditions of legal labelling. However, one important difference is that in judgment aggregation, each proposition can have two values: True or False. In labelling aggregation, on the other hand, each argument can have three values: $\myin$, $\myout$, or $\myundec$. This makes labelling aggregation be more comparable to non-binary evaluations \cite{dokow2010aggregation,dokow2010}. Considering the general framework in \cite{dokow2010}, our settings can be considered as focusing on special classes of feasible evaluations, which are the conditions imposed by the legal labelling (or other semantics). Additionally, the possible evaluations of each issue (argument, in our case) are to accept (labels as $\myin$), reject (labels as $\myout$), or be undecided (labels as $\myundec$).


In this paper, we conduct an extensive social-choice-theoretic analysis of argument evaluation semantics by means of labellings. We assume that individuals are presented with a shared argumentation framework (AF) and need to make a decision about how to evaluate this AF. Individuals are assumed to have different, but reasonable, evaluations. There can be many scenarios in which such settings are present. For example, consider a jury members that are all provided with the same information, each of them has a different opinion about these information and yet they all need to come up with a collective decision. Another example is a company board committee who need to make an informed decision. They can be all presented with the same information about the current economic status and the possible strategies, each one of them has his/her own opinion about what should be done, yet they all need to reach a collective decision. 

The paper makes three distinct contributions to the state-of-the-art in the computational modelling of argumentation. Firstly, the paper introduces the study of aggregating different individual judgments on how a given set of arguments is to be evaluated.\footnote{In fact, this idea was first introduced in \cite{rahwan:tohme:2010} for which this paper is a substantially extended and revised version. Section \ref{section:impossibility} which introduces the impossibility of good aggregation operator is significantly enhanced by adding three impossibility results. Sections \ref{section:semantics} and \ref{section:graph} are completely new. Section \ref{section:conclusion} contains more elaborate discussion of related and future work. Finally, further explanation, motivation, discussion and background is added to the other sections to improve clarity and presentation of the paper.} This requires adapting classical social-choice properties to the argumentation domain, and sometimes demands special treatment (e.g. different versions of some properties). 

The second contribution of this paper is proving the impossibility of the existence of \emph{any} aggregation operator that satisfies some minimal properties. In doing so, we show impossibility results that concern dealing with ties and producing a collectively rational evaluation of arguments. These results establish the limits of aggregation in the context of argumentation, and come in accordance with the impossibility results in the topics of aggregation such as preference aggregation \cite{arrow:1951,sen1970impossibility,muller1977equivalence,gibbard:1973,satterthwaite:1975} and judgment aggregation \cite{list:pettit:2002}. 
Hence, as is the case with other aggregation domains, the aggregation paradox in argument evaluation is an example of a more fundamental barrier. These results are important because they give conclusive answers and focus research in more constructive directions (\eg weakening the desired properties in order to avoid the paradox). Aiming to investigate possible relaxations in order to circumvent the impossibility in the context of argumentation, we broke down the \emph{Collective Rationality} postulate into sub-postulates. This helps in taking a deeper look at the distinct parts of the postulate. As a consequence, satisfying any of these parts can be used to weaken the collective rationality. 

The third contribution of this paper is an extensive analysis of an aggregation rule, namely \emph{argument-wise plurality rule}. We analyse the properties of the argument-wise plurality rule in general, and investigate whether the restriction of the domain of votes to a particular classical semantics would ensure the fulfillment of these conditions. This highlights a novel use of classical semantics, which are originally used to resolve issues in single-agent nonmonotonic reasoning. Finally, we provide graph-theoretical restrictions on argumentation frameworks under which the argument-wise plurality rule would be guaranteed to produce collectively rational outcomes.

The paper is organised as follows. In section \ref{section:background}, we start by giving a brief background on abstract argumentation systems. Sections \ref{section:aggregation}, \ref{section:postulates}, \ref{section:impossibility} and \ref{section:avoiding} focus on the problem of aggregating sets of judgments over argument evaluation. Sections \ref{section:awpr}, \ref{section:semantics}, and \ref{section:graph} focus on introducing and analysing the argument-wise plurality rule. We conclude the paper and discuss some related work in Section \ref{section:conclusion}.

\section{Background}\label{section:background}

In this section, we briefly outline key elements of abstract
argumentation frameworks. We begin with Dung's abstract
characterisation of an argumentation system \cite{dung:1995}. We restrict ourselves to finite sets of arguments. 
%
%
\begin{definition}[Argumentation framework]
    An \emph{argumentation framework} is a pair $\AF = \langle
    \calA, \rightharpoonup \rangle$ where $\calA$ is a finite set of
    arguments and $\rightharpoonup \subseteq \calA \times \calA$
    is a defeat relation. We say that an argument $a$
    \emph{defeats} an argument $b$ if $(a, b) \in
    \rightharpoonup$ (sometimes written $a ~\rightharpoonup~
    b$).
\end{definition}
For an argument $a\in \calA$, we use $a^-$ to denote the set of arguments that defeat $a$ i.e. $a^-=\{b\in\calA|b\rightharpoonup a\}$.
\begin{figure}[htbp]
    \centering
       \includegraphics[scale=1]{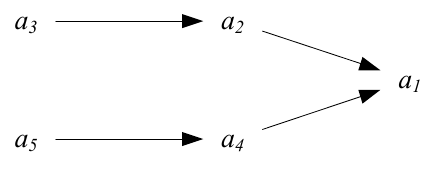}
    \caption{A simple argument graph}
    \label{fig:argument_graph1}
\end{figure}

An argumentation framework can be represented as a directed graph in
which vertices are arguments and directed arcs characterise defeat
among arguments. An example argument graph is shown in
Figure~\ref{fig:argument_graph1}. Argument $a_1$ has two
defeaters (\ie counter-arguments) $a_2$ and $a_4$, which
are themselves defeated by arguments $a_3$ and $a_5$
respectively.

There are two approaches to define semantics that assess the acceptability of arguments. One of them is extension-based semantics by Dung \cite{dung:1995}, which produces a set of arguments that are accepted together. Another equivalent labelling-based  semantics is proposed by Caminada \cite{caminada:2006,caminada:gabbay:2009}, which gives a labelling for each argument. With argument labellings, we can accept arguments (by labelling them as $\myin$), reject arguments (by labelling them as $\myout$), and abstain from deciding whether to accept or reject (by labelling them as $\myundec$). Caminada \cite{caminada:2006,caminada:gabbay:2009} established a correspondence between
properties of labellings and the different extensions. In this paper, we employ the labelling approach.

\begin{definition}[Argument Labelling]\label{definition:labelling}
Let $\AF = \langle \calA, \rightharpoonup \rangle$ be an argumentation
framework. An \emph{argument labelling} is a total function $L :
\calA \rightarrow \{\myin, \myout,$ $\myundec \}$.
\end{definition}

We write $\myin(L)$ (resp. $\myout(L)$, $\myundec(L)$) for the set of arguments that are labelled $\myin$ (resp. $\myout$, $\myundec$) by $L$. A labelling $L$ can be represented as $L=(\myin(L)$,$\myout(L)$,$\myundec(L))$.

However, labellings should follow some given conditions. A minimal reasonable condition is the \emph{conflict-freeness}.
\begin{definition}[Conflict-freeness]
A labelling $L$ satisfies conflict-freeness iff $\forall a,b \in \myin(L)$, ${\neg (a\rightharpoonup b)}$.
\end{definition}

One of the essential semantics, which satisfies conflict-freeness is the \emph{complete semantics}. We already informally defined \emph{complete} labellings via two conditions in the introduction. We find it convenient to equivalently formulate it as three conditions as follows.

\begin{definition}[Complete labelling]\label{definition:CompLabelling}
Let $\AF = \langle \calA, \rightharpoonup \rangle$ be an argumentation
framework. A \emph{complete labelling} is a total function $L :
\calA \rightarrow \{\myin, \myout,$ $\myundec \}$ such that:
\begin{itemize}
    \item $\forall a \in \calA ~:~ \text{if } L(a) = \myin
    \text{ then } \forall b \in \calA ~:~ (b
    \rightharpoonup a \Rightarrow L(b) = \myout)$; 
    
    \item $\forall a \in \calA ~:~ \text{if } L(a) = \myout
    \text{ then } \exists b \in \calA ~\text{s.t.}~ (b
    \rightharpoonup a \wedge L(b) = \myin)$; and
    
    \item $\forall a \in \calA ~:~ \text{if } L(a) = \myundec
    \text{ then }$ 
    \begin{itemize}
    \item $\exists b \in \calA ~:~ (b
    \rightharpoonup a \wedge L(b) = \myundec); ~\text{and}~$
		\item $\not \exists b \in \calA ~:~ (b
    \rightharpoonup a \wedge L(b) = \myin)$ 
    
		\end{itemize}
\end{itemize}
We will use $\comp{\AF}$ to denote the set of all \emph{complete} labellings for $\AF$.
\end{definition}

As an example, consider the following.
\begin{example}
Consider the graph in Figure \ref{fig:argument_graph2}. Here, we have three complete labellings: $L^G=(\{a_3\},\{\},\{a_1,a_2\})$, $L_1=(\{a_1,a_3\},\{a_2\},\{\})$, and $L_2=(\{a_2,a_3\},\{a_1\},\{\})$.
\end{example}

\begin{figure}[htbp]
  \centering
  \includegraphics[scale=1]{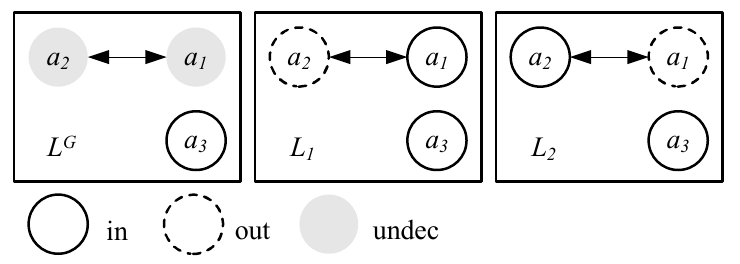}
  \caption{Graph with three complete labellings.}
  \label{fig:argument_graph2}
\end{figure}

In addition to the \emph{complete} labelling, there are other semantics which assume further conditions. 

\begin{definition}[Other Labelling Semantics]
Let $\AF = \langle \calA, \rightharpoonup \rangle$ be an argumentation framework. Let $L: \calA \rightarrow \{\myin,\myout,\myundec\}$ be a complete labelling.
\begin{itemize}
\item $L$ is a grounded labelling if and only if $\myin(L)$ is minimal, or equivalently $\myout(L)$ is minimal, or equivalently $\myundec(L)$ is maximal (w.r.t set inclusion) among all \emph{complete} labellings.
\item $L$ is a preferred labelling if and only if $\myin(L)$ is maximal, or equivalently $\myout(L)$ is maximal (w.r.t set inclusion) among all \emph{complete} labellings.
\item $L$ is a semi-stable labelling if and only if $\myundec(L)$ is minimal (w.r.t set inclusion) among all \emph{complete} labellings.
\item $L$ is a stable labelling if and only if $\myundec(L) = \emptyset$. 
\end{itemize}
\end{definition}

Note that the grounded labelling is always unique, and stable labellings might not exist. Consider the following example.

\begin{example}
Consider the graph in Figure \ref{fig:argument_graph2}. Here, we have the \emph{grounded} labelling is $L^G=(\{a_3\},\{\},\{a_1,a_2\})$. We have only two \emph{preferred} labellings: $L_1=(\{a_1,a_3\},\{a_2\},\{\})$, and $L_2=(\{a_2,a_3\},\{a_1\},\{\})$. These are also the only \emph{stable} and \emph{semi-stable} labellings for this framework.
\end{example}

Clearly, for any $\AF$, $\stab{\AF} \subseteq \semi{\AF} \subseteq \pref{\AF} \subseteq \comp{\AF}$, and $\grnd{\AF} \subseteq \comp{\AF}$, where $\stab{\AF}$, $\semi{\AF}$, $\pref{\AF}$, and $\grnd{\AF}$ refer to the set of \emph{stable}, \emph{semi-stable}, \emph{preferred}, and \emph{grounded} labellings for $\AF$. We refer to the previous semantics as \emph{classical semantics}. There exist other semantics which we do not consider in this work.


\section{Aggregation of Argument Labellings}\label{section:aggregation}

To date, most analyses inspired by Dung's framework have focused on analysing and comparing  the properties of various types of extensions/labellings (\ie semantics) \cite{baroni:giacomin:2007}. 
The question is, therefore, whether a particular type of labelling is appropriate for a particular type of reasoning task in the presence of conflicting arguments.

In contrast with most existing work on Dung frameworks, our concern here is with multi-agent systems. Since each labelling captures a particular rational point of view, we ask the following question: \emph{Given an argumentation framework and a set of agents, each with a legitimate subjective evaluation of the given arguments, how can the agents reach a collective compromise on how to evaluate those arguments?}

Thus, the problem we face is that of \emph{judgment aggregation} \cite{list:pettit:2002} in the context of argumentation frameworks. 
This problem can be formulated as a set of individuals that collectively decide how an argumentation framework $\AF = \langle \calA, \rightharpoonup \rangle$ must be labelled. 
\begin{definition}[Labelling aggregation problem]
Let $\Ag=\{1,\ldots,n\}$ be a finite non-empty set of agents, and $\AF = \langle \calA, \rightharpoonup \rangle$ be an argumentation framework. A labelling aggregation problem is a pair $\calLAP= \langle \Ag, \AF \rangle$.
\end{definition}
Each individual $i \in \Ag$ has a labelling $L_i$ which expresses the evaluation of $\AF$ by this individual. A labelling profile is an $|\Ag|$-tuple of labellings.
\begin{definition}[Labelling profile]
Let $\calLAP= \langle \Ag, \AF \rangle$ be a labelling aggregation problem. We use $\calL=(L_1,\ldots,L_n) \in \mathbf{L}(\AF)^{|\Ag|}$ to denote a labelling profile, where $\mathbf{L}(\AF)$ is the class of labellings of $\AF$. Additionally, we use $\calL(a)$ to denote the labelling profile (i.e. an $|\Ag|$-tuple) of an argument $a \in \calA$ i.e. $\calL(a)=(L_1(a),\ldots,L_n(a))$.
\end{definition}
The aggregation of individuals' labellings can be defined as a partial function.\footnote{We state that the function is partial to allow for cases in which collective judgment may be undefined (e.g. when there is a tie in voting).}
\begin{definition}[Aggregation function]
Let $\calLAP= \langle \Ag, \AF \rangle$ be a labelling aggregation problem. An aggregation function for $\calLAP$ is a function $F: \mathbf{L}(\AF)^{n} \rightarrow \mathbf{L}(\AF)$. 
\end{definition}
For each $a \in \calA$, $[F(\calL)](a)$ denotes the collective label assigned to $a$, if $F$ is defined for $\calL=(L_{1}, \ldots,L_{n})$.

\section{Desirable Properties of Aggregation Operators}\label{section:postulates}

Aggregation involves comparing and assessing different points of view. There are, of course, many ways of doing this, as extensively
discussed in the literature of Social Choice Theory \cite{gartner:2006}. In this literature, a consensus on some normative ideals has been reached, identifying what a `fair' way of adding up votes should be. So for instance, if everybody agrees, the outcome must reflect that
agreement; no single agent can impose her view on the aggregate; the aggregation should be performed in the same way in each possible
case, $\etc$ These informal requirements can be formally stated as properties that $F$ should satisfy \cite{list:pettit:2002, dietrich:2007}. In all of the following postulates, it is assumed that a fixed labelling aggregation problem $\calLAP= \langle \Ag, \AF \rangle$ is given. The postulates can be grouped as follows:\footnote{This style of presentation of postulates was inspired by \cite{GrandiPhD12} which is on binary aggregation.}\\

\noindent{\em Group 1: Domain and co-domain postulates}\\

In judgment aggregation, two postulates that are commonly assumed are those of {\em Universal Domain} and {\em Collective Rationality}. The former requires that any profile of labellings chosen from a pre-specified set of {\em feasible} labellings can be used as input to $F$ and $F$ will return an answer. The question is: what do we take to be the set of feasible labellings in our setting? This depends on which semantics we assume is being used. Theoretically we can have a different version of \emph{Universal Domain} for each semantics. However since \emph{complete} semantics represent reasonable and self-defending points of views, it represents the best counterpart for the logical consistency in judgment aggregation:

\begin{quote}
{\bf Universal Domain} $F$ can take as input all profiles $\calL = (L_1,\ldots,L_n)$ such that $\calL\in \comp{\AF}^n$
\end{quote}
However, in Subsection \ref{subsection.othersem} we will use other semantics as a domain for $\calL$.

Similarly we could have a different version of \emph{Collective Rationality} - one for each semantics - stating that the output of the aggregation should also be feasible. Again, since we focus on complete semantics, we focus on the following version:

\begin{quote}
{\bf Collective Rationality} 
For all profiles $\calL$ such that $F(\calL)$ is defined, $F(\calL) \in \comp{\AF}$.
\end{quote}
Later, in Section \ref{section:avoiding}, we will break this postulate down into further constituents.\\ 

\noindent{\em Group 2: Fundamental postulates}\\

Next we come to the standard property that forms the cornerstone of the usual impossibility results in judgment aggregation. It says the collective label of an argument depends only on the votes on that argument, independent of the other arguments.

\begin{quote}
{\bf Independence} For any two profiles $\calL =  (L_1,\ldots,L_n)$, $\calL' = (L'_1, \ldots, L'_n)$ such that $F(\calL)$ and $F(\calL')$ are defined, and for all $a \in \calA$, if $L_i(a) = L'_i(a)$ for all $i \in \Ag$, then $[F(\calL)](a) = [F(\calL')](a)$. 
\end{quote}

\noindent
The effect of \emph{Independence} is that aggregation is done ``argument-by-argument". To be slightly more precise, each argument $a \in \calA$ essentially has its own aggregation operator $I_a$ associated to it, that takes an $n$-tuple of labels $\mathbf{x} = (l_1, \ldots, l_n)$ as input (representing the ``vote" of each agent on the label of $a$) and returns another label $I_a(\mathbf{x})$ as output (the ``collective label") of $a$. Then $[F(\calL)](a) = I_a((L_1(a), \ldots, L_n(a)))$. Note that the necessity of \emph{Independence} is questionable in our settings because of the dependencies between arguments that come already encoded in the form of the attack relation. Nevertheless, it is usually investigated in the judgment aggregation and preference aggregation literature because of its role in analysing strategy-proofness. Though the relation between \emph{Independence} and strategy-proofness is not established yet in our settings, our task in this paper is to stick close to the methodology in judgment aggregation, and there it is often assumed.

Next, we have \emph{Anonymity}, which says the identity of which agent submits which labelling is irrelevant.
\begin{quote}
{\bf Anonymity}
For any profile $\calL=(L_1,\ldots,L_n)$, if $\calL' = (L_{\rho(1)}, \ldots, L_{\rho(n)})$ for some permutation $\rho$ on $\Ag$, and $F(\calL)$ and $F(\calL')$ are both defined, then $F(\calL) = F(\calL')$.
\end{quote}

\noindent
If we add \emph{Anonymity} to \emph{Independence}, then it means the outputs of the functions $I_a$ described above depend only on the \emph{number} of votes that each label gets in $\mathbf{x}$. Essentially it means $I_a$ outputs a collective label just taking as input the triple $(\#\myin,\#\myout,\#\myundec)$ of numbers denoting, respectively, the number of votes for $\myin$, $\myout$ and $\myundec$ in $\mathbf{x}$. 

\begin{proposition}
Let $F$ be an aggregation operator. Then $F$ satisfies both \emph{Independence} and \emph{Anonymity} iff for each $a \in \calA$ there exists a function $I_a: \mathbb{N}^3 \rightarrow \{\myin, \myout, \myundec \}$ such that, for all $\calL$ we have $[F(\calL)](a) = I_a(\#\myin,\#\myout,$ $\#\myundec)$. 
\end{proposition}

\begin{proof}[Outline]
The ``if'' case is straightforward, since permuting the rows does not change the vote distribution and so \emph{Anonymity} will hold. \emph{Independence} is also clear. 

For the ``only if'' case, \emph{Independence} gives us the existence of the function $I_a$ such that $[F(\calL)](a) = I_a(L_1(a), \ldots, L_n(a))$ and then \emph{Anonymity} implies that two vectors that have the same vote distribution will give the same results, so we can set $I_a(\#\myin,\#\myout,\#\myundec) = I_a(L_1(a), \ldots, L_n(a))$ where $(L_1(a), \ldots, L_n(a))$ is any vote which has $(\#\myin,\#\myout,\#\myundec)$ as its distribution.
\end{proof}

A weakening of \emph{Anonymity} is \emph{Non-Dictatorship}:\footnote{Since a violation of the latter would imply a violation of the former.}

\begin{quote}
{\bf Non-Dictatorship}
There is no $i \in \Ag$ such that, for every profile $\calL = (L_1, \ldots, L_n)$ for which $F(\calL)$ is defined, we have $F(\calL) = L_i$.
\end{quote}

\noindent{\em Group 3: \emph{Unanimity} postulates}\\

Next we move to \emph{Unanimity}, and some other postulates related to it.

\begin{quote}
{\bf Unanimity}
If $\calL$ is such that $F(\calL)$ is defined and there exists some $L$ s.t.\ $L_i = L$ for all $i \in \Ag$, then $F(\calL) = L$. 
\end{quote}

\noindent
This postulate is also familiar from judgment aggregation, but the move to 3-valued labellings rather than the 2 usually seen in judgment aggregation opens up the possibility to define other variants of \emph{Unanimity}, one of which is used by Dokow and Holzman \cite{dokow2010}, called \emph{Supportiveness}:

\begin{quote}
{\bf Supportiveness}
For any profile $\calL$ such that $F(\calL)$ is defined, and for all $a \in \calA$, there exists $i \in \Ag$ such that $[F(\calL)](a) = L_i(a)$.
\end{quote}

\noindent
\emph{Supportiveness} says that, for each argument $a$ and label $l$, the collective judgment cannot be set to $l$ without at least one agent voting for that $l$. Clearly \emph{Supportiveness} implies \emph{Unanimity}. 

It might seem natural to have the collective label of an argument as $\myundec$ even when nobody votes for it, if we interpret $\myundec$ as a halfway label between $\myin$ and $\myout$. Then if half the agents say $\myin$ and the other half says $\myout$ then $\myundec$ might be a reasonable compromise. Given this, a weaker version of \emph{Supportiveness} that only applies to $\myin$ and $\myout$ can be defined. We call it \emph{$\myin/\myout$-Supportiveness}.\\

\begin{quote}
{\bf $\myin/\myout$-Supportiveness}
For any profile $\calL$ such that $F(\calL)$ is defined, and for all $a \in \calA$, if $[F(\calL)](a) \neq \myundec$ then there exists some agent $i$ such that $[F(\calL)](a) = L_i(a)$.\\
\end{quote}

\noindent{\em Group 4: \emph{Systematicity} postulates}\\

Now we come to the \emph{Systematicity} postulates which deal with neutrality issues across arguments and labels. We can list two variants, both of which imply \emph{Independence}. We start with the stronger version:

\begin{quote}
{\bf Strong Systematicity}
For any two profiles $\calL = (L_1,\ldots,L_n)$ and $\calL' = (L_1', \ldots, L_n')$ such that $F(\calL)$ and $F(\calL')$ are defined, and for all $a,b \in \calA$, and for every permutation $\rho$ on the set of labels $\{ \myin, \myout, \myundec\}$, if $\forall i \in \Ag:$ $L_i(a) = \rho(L_i'(b))$, then $[F(\calL)](a) = \rho([F(\calL')](b))$.
\end{quote}

To illustrate \emph{Strong Systematicity}, consider the example in Figure \ref{fig:strsys}. We have the following three labellings: $L_1 = (\{a\}, \{b\},\{\})$, $L_2 = (\{b\}, \{a\},\{\})$, $L_3 = (\{\}, \{\},\{a,b\})$.

\begin{figure}[htbp]
	\begin{center}
  \includegraphics[scale=1.2]{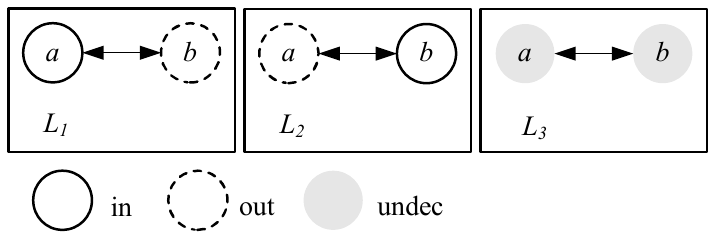}
	\end{center}
  \caption{An example illustrating \emph{Strong Systematicity}.}\label{fig:strsys}
\end{figure}

Consider the profiles $\calL=(L_1,L_1,L_2,L_3)$ and $\calL'=(L_3,L_3,L_2,L_1)$. Then, $\calL(a)=(\myin,\myin,\myout,\myundec)$ and $\calL'(b)=(\myundec,\myundec,\myin,\myout)$. Let $\rho$ be the permutation on labels such that $\rho(\myin)=\myundec$, $\rho(\myout)=\myin$, and $\rho(\myundec)=\myout$. Then, we can see that in this example $\forall i \in \Ag :$ $L'_i(b)=\rho(L_i(a))$. \emph{Strong Systematicity} requires that $[F(\calL')](b) = \rho([F(\calL)](a))$.

\noindent
The postulate forces us to give an even-handed treatment to the labels $\myin$, $\myout$ and $\myundec$ (in addition to treating each argument independently and similarly). This makes sense if we consider $\myin$, $\myout$ and $\myundec$ as three independent labels. However, one might be tempted to consider $\myundec$ as a middle label between $\myin$ and $\myout$. Hence, the equal treatment might not be desirable in this case. One might suggest a version of \emph{Systematicity} that treats $\myin$ and $\myout$ equally. Following, we define this version (which we call $\myin/\myout$-\emph{Systematicity}). \\

\begin{quote}
{\bf $\myin/\myout$-Systematicity}
For any two profiles $\calL = (L_1,\ldots,L_n)$ and $\calL' = (L_1', \ldots, L_n')$ such that $F(\calL)$ and $F(\calL')$ are defined, and for all $a,b \in \calA$, and for every $\myundec$-preserving permutation $\rho$ on the set of labels $\{ \myin, \myout, \myundec\}$ ($\ie \rho(\myundec)=\myundec$), if $\forall i \in \Ag:$ $L_i(a) = \rho(L_i'(b))$, then $[F(\calL)](a) = \rho([F(\calL')](b))$.\\
\end{quote}
$\myin/\myout$-\emph{Systematicity} lies in the middle between \emph{Strong Systematicity} and the following version of \emph{Systematicity} which can be obtained by restricting the class of permutations, until we only consider the identity. \\

\begin{quote}
{\bf Weak Systematicity}
For any two profiles $\calL = (L_1,\ldots,L_n)$ and $\calL' = (L_1', \ldots, L_n')$ such that $F(\calL)$ and $F(\calL')$ are defined, and for all $a,b \in \calA$, if $\forall i \in \Ag:$ $L_i(a) = L_i'(b)$, then $[F(\calL)](a) = [F(\calL')](b)$.
\end{quote}

\noindent 
Clearly \emph{Independence} follows from \emph{Weak Systematicity} by just setting $a=b$. If we strengthen \emph{Independence} to \emph{Weak Systematicity} then the functions $I_a$, mentioned earlier, are identical for all arguments.\\

\noindent{\em Group 5: \emph{Monotonicity} postulates}\\

Our final group relates to \emph{Monotonicity}.

\begin{quote}
{\bf Monotonicity}
Let $l_a \in \{\myin, \myout, \myundec\}$ be such that given two profiles $\calL = (L_{1}, \ldots, L_{i},$ $\ldots, L_{i + k}, \ldots, L_{n})$ and $\calL'=(L_{1}, \ldots, L'_{i}, \ldots, L'_{i + k}, \ldots, L_{n})$ (differing only in the labellings of agents $i, i+1,\ldots, i+k$) such that $F(\calL)$ and $F(\calL')$ are defined, where $i \in \{1,\ldots,n\}$ and $k \in \{0,\ldots, n-i\}$, if $L_{j}(a) \neq l_a$ while $L^{'}_{j}(a) = l_a$ for all $j \in \{i,\ldots ,i+k\}$, then $[F(\calL)](a) = l_a$ implies that $[F(\calL')](a) = l_a$.
\end{quote}
\noindent
\emph{Monotonicity} states that if a set of agents switch their label of argument $a$ to the collective label of $a$ then the collective label of $a$ remains the same. Similar to \emph{Supportiveness} and \emph{Systematicity}, a weaker version of \emph{Monotonicity} that only apply to $\myin$ and $\myout$ can be defined. We call it \emph{$\myin/\myout$-Monotonicity}.

\begin{quote}
{\bf $\myin/\myout$-Monotonicity}
Let $l_a \in \{\myin, \myout\}$ be such that given two profiles $\calL = (L_{1}, \ldots, L_{i},$ $\ldots, L_{i + k}, \ldots, L_{n})$ and $\calL'=(L_{1}, \ldots, L'_{i}, \ldots, L'_{i + k},$ $\ldots, L_{n})$ (differing only in the labellings of agents $i, i+1,\ldots, i+k$) such that $F(\calL)$ and $F(\calL')$ are defined, where $i \in \{1,\ldots,n\}$ and $k \in \{0,\ldots, n-i\}$, if $L_{j}(a) \neq l_a$ while $L^{'}_{j}(a) = l_a$ for all $j \in \{i,\ldots ,i+k\}$, then  $[F(\calL)](a) = l_a$ implies that $[F(\calL')](a) = l_a$.\\
\end{quote}


\section{The Argument-Wise Plurality Rule}\label{section:awpr}

An obvious candidate aggregation operator to check out is the {\em plurality}
voting operator $M$. In this section, we analyse a number of key properties of this operator. Intuitively, for each argument, it selects the label that appears most frequently in the individual labellings.

\begin{definition}[Argument-Wise Plurality Rule (AWPR)]
Let $\AF =\langle \calA, \rightharpoonup \rangle$ be an argumentation framework. Given any argument $a \in \calA$ and any profile $\calL =(L_{1}, \ldots, L_{n})$, it holds that $[M(\calL)](a)=l_a \in \{\myin,\myout,\myundec\}$ iff 
$$|\{i: L_{i}(a) = l_a\}| >
    \max_{l'_a \neq l_a}|\{i: L_{i}(a) = l'_a \}|$$
\end{definition}

Note that $M$ is defined for all profiles that cause no ties, i.e. $M(\calL)$ is defined iff there does not exist any argument $a\in\calA$ for which we have at least two labels $l_a$ and $l'_a$ with $l_a \neq l'_a$ and 
\[|\{i: L_{i}(a) = l_a\}| = |\{i: L_{i}(a) = l'_a\}| = \max_l|\{i: L_{i}(a) = l \}|\] 

One can directly notice that AWPR violates \emph{Universal Domain}, because it is not defined for all profiles in $\comp{\AF}$.


\begin{example}[Three Detectives (cont.)]
Continuing on Example \ref{example:murder_aggregation}, applying the argument-wise plurality rule, we have $[M((L_1, L_2, L_3))](a_1)=$ $\myout$, $[M((L_1, L_2, L_3))](a_2) =$ $\myin$, and $[M((L_1, L_2, L_3))](a_3) =$ $\myout$.
\end{example}

\subsection{Properties of Argument-Wise Plurality Rule}

We now analyse whether AWPR satisfies the properties listed above.

\begin{proposition}\label{theorem:1to7}
The argument-wise plurality rule operator $M$ satisfies \emph{Supportiveness}, \emph{Anonymity}, \emph{Strong Systematicity}, and \emph{Monotonicity}.
\begin{proof}
In this proof, the considered profiles are restricted to those for which $[M(\calL)]$ is defined.

\begin{itemize}

\item {\em Supportiveness}: consider any profile $\calL = (L_1,\ldots,L_n)$. Suppose, towards a contradiction, that for some argument $a$, there exists no agent $i$ such that $L_i(a) = l_a$ where $l_a=[M(\calL)](a)$. Then $|\{i: L_{i}(a) = l_a\}|=0$. But, $|\{i: L_{i}(a) = l_a\}|>\max_{l'_a \neq l_a}|\{i: L_{i}(a) = l'_a\}|>0$ (the last inequality holds since $\Ag$ is non-empty). Contradiction. 


\item {\em Anonymity}: consider any profile $\calL = (L_1,\ldots,L_n)$. $[M(\calL)](a) = l_a$ if and only if $|\{i:
L_{i}(a)=l_a\}| > \max_{l'_a \neq l_a}|\{i:
L_{i}(a) = l'_a\}|$ if and only if $|\{\rho(i):
L_{\rho(i)}(a)=l_a\}| > \max_{l'_a \neq l_a}|\{\rho(i):
L_{\rho(i)}(a)=l'_a\}|$, which is equivalent to\\
$[M((L_{\rho(1)}, \ldots, L_{\rho(i)}, \ldots,
L_{\rho(n)}))](a) = l_a$.

\item {\em Strong Systematicity}: consider, for any two profiles $\calL = (L_1,\ldots,L_n)$ and $\calL' = (L'_1,\ldots,L'_n)$, and for any $a, b \in \calA$, the permutation $\rho: \{ \myin, \myout, \myundec\}$ $\rightarrow \{ \myin, \myout, \myundec\}$. Suppose, towards a contradiction, that for any $i$, $L_i(a) = \rho(L_i'(b))$, and $[M(\calL)](a) =l_a$ but $\rho(M(\calL')[b]) \neq \rho(l_a)$. But then,$|\{i: L_{i}(a)=l_a\}| = |\{i: L^{'}_{i}(b)=\rho(l_a)\}|$ while for any $l'_a \neq l_a$, $|\{i: L_{i}(a) = l'_a\}| = |\{i: L^{'}_{i}(b) = \rho(l'_a)\}|$. So, if $|\{i: L_{i}(a)=l_a\}| > \max_{l'_a \neq l_a}|\{i: L_{i}(a) = l'_a\}|$ then, we have $|\{i: L^{'}_{i}(b)=\rho(l_a)\}| > \max_{l'_a \neq l_a}|\{i: L^{'}_{i}(b) = \rho(l'_a)\}|$ as well. Contradiction. 




\item {\em Monotonicity}: Consider the following two profiles $\calL = (L_{1}, \ldots, L_{i}, \ldots, L_{i + k},$ $\ldots, L_{n})$ and $\calL'=(L_{1}, \ldots, L'_{i}, \ldots, L'_{i + k}, \ldots, L_{n})$ (differing only in the labellings of agents $i, i+1,\ldots, i+k$) where $i \in \{1,\ldots,n\}$ and $k \in \{0,\ldots, n-i\}$. Suppose, towards a contradiction, that for $a \in \calA$ and a label $l_a$ we have that $L_{h}(a)\neq l_a$ while $L^{'}_{h}(a)=l_a$ for all $h \in \{i,\ldots ,i+k\}$, and we have that $[M(\calL)](a) = l_a$ while
$[M(\calL')](a) \neq l_a$. But then, $|\{j: L_{j}(a) = l_a\}| > \max_{l'_a \neq l_a}|\{j: L_{j}(a) = l'_a\}|$ in the profile $\calL$ while in the profile $(\hat{L}_{1}, \ldots,\hat{L}_{n})$$\equiv$$\calL'$, 
\noindent we have  $ \{j: \hat{L}_{j}(a) = l_a\}= \{j: L_{j}(a) = l_a\} \cup \{i,\ldots,i+k\}$ and $\{j: \hat{L}_j(a) = l'_a\} \subseteq \{j:L_j(a) =l'_a \}$ for every other labelling $l'_a$. Then $|\{j: \hat{L}_{j}(a) = l_a\}| > \max_{l'_a \neq l_a}|\{j: \hat{L}_{j}(a) = l'_a\}|$. Contradiction.


\end{itemize}
\end{proof}
\end{proposition}

\begin{corollary}\label{cor.postu}
The argument-wise plurality rule operator $M$ satisfies \emph{Unanimity}, \emph{Weak Systematicity}, \emph{Independence}, and \emph{Non-Dictatorship}.
\begin{proof}
{\em Weak Systematicity} and {\em Independence} follow from \emph{Strong Systematicity}, {\em Unanimity} follows from \emph{Supportiveness}, and
{\em Non-Dictatorship} follows from \emph{Anonymity}. 
\end{proof}
\end{corollary}

Despite all these promising results, it turns out that plurality operator violates \emph{Universal Domain} and \emph{Collective Rationality} postulates. The violation of \emph{Universal Domain} is because AWPR is not defined for profiles that cause ties, which means that it cannot take as input every possible profile $\calL \in \comp{\AF}^n$. However, a weaker version of \emph{Universal Domain} can be defined.

\begin{quote}
{\bf No-Tie Universal Domain}
An aggregation operator $F$ can take as input all profiles $\calL = (L_1,\ldots,L_n)$ such that $\calL$ does not cause a tie and $\calL\in \comp{\AF}^n$.
\end{quote}  

Since there are no restrictions (other than having no ties) on how labellings are defined, AWPR satisfies \emph{No-Tie Universal Domain}. Note that one might be tempted to make AWPR satisfy \emph{Universal Domain} by adding a deterministic\footnote{The use of a non-deterministic tie-breaking rule has its own issues too, such as producing different outcomes given the same profile.} tie-breaking rule to deal with ties. However, as we show in the next section, the use of any tie-breaking rule would result in violating \emph{Anonymity}, and/or \emph{Strong Systematicity}. While the violation of \emph{Universal Domain} represents a minor inconvenience that can be justified, the violation of \emph{Collective Rationality} poses a serious issue as the collective decision is usually expected to be reasonable. The following example shows how AWPR violates \emph{Collective Rationality}.

\begin{example}\label{example:maincounter}
Suppose argument $c$ has two defeaters, $a$ and $b$, and argument $a$ (resp. $b$) defeats and is defeated by argument $a'$ (resp. $b'$). Suppose we have $3$ agents, with votes as shown in Figure \ref{fig:counter_example}. We have $[M(\calL)](c)= \myout$, but it is not the case that $[M(\calL)](a)= \myin$ or $[M(\calL)](b)= \myin$.
\end{example}

\begin{figure}[htb]
  \centering
  \includegraphics[scale=0.8]{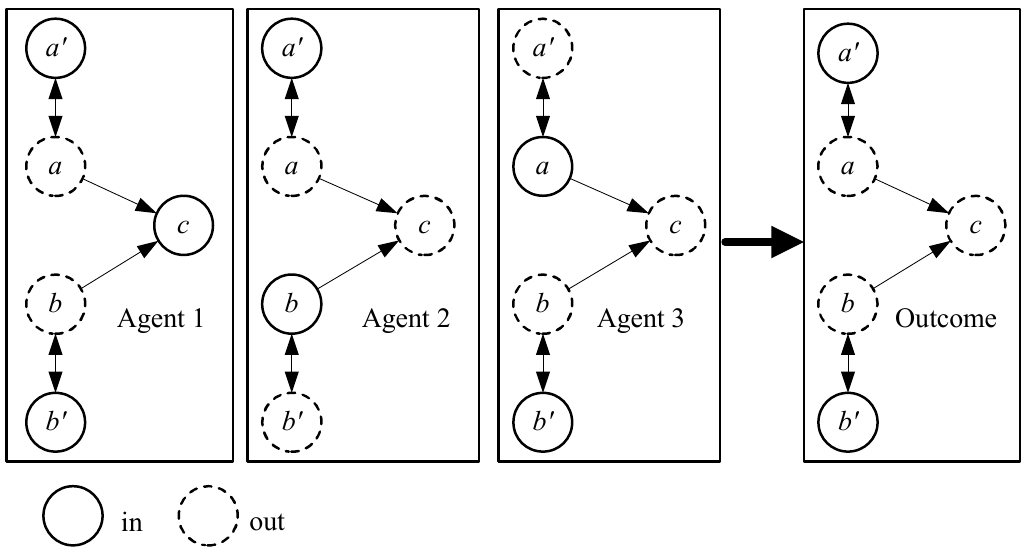}
  \caption{Counterexample to \emph{Collective Rationality}}\label{fig:counter_example}
\end{figure}

Interestingly, the above counterexample demonstrates a variant of the discursive dilemma \cite{list:pettit:2002} in the context of argument evaluation, which itself is a variant of the well-known Condorcet paradox.

\section{The Impossibility of Good Aggregation Operators}\label{section:impossibility}

In the previous section, we analysed a particular judgment aggregation operator (namely, argument-wise plurality rule). We showed that while it satisfies most key properties, it fails to satisfy \emph{Universal Domain} and \emph{Collective Rationality}. In this section, we show a couple of impossibility results that involve these two postulates. The following result shows that introducing a tie-breaking rule to satisfy \emph{Universal Domain} would result in violating \emph{Anonymity} and/or \emph{Strong Systematicity}.

\begin{subequations}

\begin{theorem}\label{thm:impossibility1}
There exists an argumentation framework $\AF$ such that, for any set of agents whose cardinality is divisible by three, there exists no labelling aggregation operator satisfying \emph{Universal Domain}, \emph{Anonymity} and \emph{Strong Systematicity}.

\begin{proof}
It is enough to assume an $\AF$ that contains at least one argument $a$ which can feasibly take on any label, i.e. there exist complete labellings $L_\myin, L_\myout$ and $L_\myundec$ over $\AF$ such that $L_\myin(a) = \myin$, $L_\myout(a) = \myout$ and $L_\myundec(a) = \myundec$.
Divide $n$ agents into 3 groups $G_1$, $G_2$, $G_3$ of equal size. By \emph{Universal Domain}, all profiles consisting of legal labellings are valid input. Assume a profile in which everyone in $G_1$ provides labelling $L_\myin$, everyone in $G_2$ provides  $L_\myout$ and everyone in $G_3$ provides  $L_\myundec$. For now let's denote this profile by $\calL = ([G_1:L_\myin], [G_2:L_\myout], [G_3:L_\myundec])$. Now, assume for contradiction that $F$ is an aggregation operator for $\AF$ satisfying \emph{Universal Domain}, \emph{Anonymity} and \emph{Strong Systematicity}. Let $\rho:\{\myin,\myout,\myundec\} \rightarrow \{\myin,\myout,\myundec\}$ be any permutation on the set of labels such that $\rho(l) \neq l$ for all labels $l$ (for instance, $\rho(\myin) = \myout$, $\rho(\myout) = \myundec$, $\rho(\myundec) = \myin$), and let $\calL'$ denote the profile $([G_1:L_{\rho(\myin)}], [G_2:L_{\rho(\myout)}], [G_3:L_{\rho(\myundec)}])$. Since $L'_i(a)=\rho(L_i(a))$ for all $i \in\Ag$, \emph{Strong Systematicity} implies $[F(\calL')](a) = \rho([F(\calL)](a))$. However, we chose $\rho$ s.t. $\rho(l) \neq l$. Hence, $[F(\calL')](a) \neq [F(\calL)](a)$. But \emph{Anonymity} implies $[F(\calL)](a) = [F(\calL')](a)$. Contradiction. Hence no such $F$ can exist.

\end{proof}
\end{theorem}
\end{subequations}

The previous result can be read in two ways: First, the AWPR cannot be made to satisfy \emph{Universal Domain} without violating \emph{Strong Systematicity} or \emph{Anonymity}. Second, there exists no aggregation operator at all that satisfies \emph{Universal Domain}, \emph{Strong Systematicity} and \emph{Anonymity}.

Note that the previous theorem was stated for a set of agents divisible by three. Essentially, three-way ties would only happen if the cardinality of the agents is divisible by three (since there are only three possible labels for each argument, and each individual has to submit one label for each argument). Hence, one might wonder whether we could rule out the possibility of three-way ties, by assuming $n$ cannot be a multiple of three.\footnote{It was shown in \cite{moulin2014strategy} that \emph{Anonymity}, \emph{Neutrality} (a weaker version of \emph{Strong Systematicity}) and \emph{Resolution} can be satisfied together if and only if the number of alternatives cannot be written as the sum of non-trivial dividers of the number of voters. \emph{Resolute} rules always produce a single outcome, so it resembles \emph{No-Tie Universal Domain}. Also, in our settings, the number of candidates is three. So this result says that we can have these postulates together if the number of voters is not a multiple of three.} However, with even number of agents, we can show that there is still a large class of $\AF$s which do not have an operator satisfying those three postulates without violating \emph{Collective Rationality}.

\begin{theorem}
There exists an argumentation framework AF such that, for any set of agents of even cardinality, there exists no labelling aggregation operator satisfying \emph{Universal Domain}, \emph{Anonymity}, \emph{Strong Systematicity} and \emph{Collective Rationality}.

\begin{proof}
It is enough to assume an AF that contains at least one argument $a$ that can feasibly take on just two out of the three possible labels. For concreteness suppose $a$ can only take on labels $\myout$ and $\myundec$ (An example of such a framework and an argument can be seen in the proof of Theorem \ref{thm:impossibility2} below, in which $c$ can only be either $\myout$ or $\myundec$). Let $L_\myundec$ and $L_\myout$ be two complete labellings such that $L_\myundec(a) = \myundec$ and $L_\myout(a) = \myout$. Divide the agents into two groups $G_1$, $G_2$ of equal size. By \emph{Universal Domain}, all profiles consisting of legal labellings are valid input, so assume a profile in which everyone in $G_1$ provides labelling $L_\myundec$ and everyone in $G_2$ provides $L_\myout$. Denote the resulting profile by $\calL = ([G_1: L_\myundec], [G_2: L_\myout])$ and assume for contradiction that $F$ is an aggregation operator for this AF that satisfies \emph{Universal Domain}, \emph{Anonymity}, \emph{Strong Systematicity} and \emph{Collective Rationality}. Let $\rho$ be the permutation that swaps $\myundec$ and $\myout$, i.e., $\rho(\myundec) = \myout$ and $\rho(\myout) = \myundec$, and let $\calL' = ([G_1: L_\myout], [G_2: L_\myundec])$.\footnote{Note here that all labellings in the profile $\calL'$ are still \emph{complete} labellings. This is because $\rho$ does not uniformly exchange all labels in a given labelling, it is just a permutation on the set of labels.} By \emph{Anonymity} we know $[F(\calL)](a) = [F(\calL')](a)$. Then it cannot be that $[F(\calL)](a) = \myundec$, for if so then \emph{Strong Systematicity} would imply $[F(\calL')](a) = \rho(\myundec) = \myout \neq [F(\calL)](a)$, and similarly it cannot be that  $[F(\calL)](a) = \myout$. Thus we must have $[F(\calL)](a) = \myin$. But by \emph{Collective Rationality} $[F(\calL)](a) \in \{\myundec, \myout\}$. Contradiction.
\end{proof}
\end{theorem}

The careful reader can realise that \emph{Collective Rationality} can be substituted with \emph{Supportiveness} in the previous theorem. As for the proof, the last sentence becomes: ``Thus we must have $[F(\calL)](a) = \myin$. But by \emph{Supportiveness} $[F(\calL)](a) \in \{\myundec, \myout\}$. Contradiction''.

However, one might argue that \emph{Strong Systematicity} is quite a strong condition. Treating $\myin$, $\myout$, and $\myundec$ differently can be tolerated. Then, it is interesting to ask: ``Does there exist an operator that satisfies \emph{Universal Domain}, \emph{Weak Systematicity}, and \emph{Anonymity}?''. The answer for this question is positive. Consider a modified version of the AWPR that deals with ties by labelling every argument that has a tie with $\myundec$. One can show that this operator satisfies these three properties together. However, this operator still violates \emph{Collective Rationality} (Example \ref{example:maincounter} holds as a counterexample). In fact, we show that any operator that satisfies \emph{Universal Domain}, \emph{Weak Systematicity}, and \emph{Anonymity}, would violate either \emph{Collective Rationality} or \emph{Unanimity}.

\begin{subequations}
\begin{theorem}\label{thm:impossibility2}
There exists an argumentation framework $\AF$ such that, for any set of agents of even cardinality, there exists no labelling aggregation operator satisfying \emph{Universal Domain}, \emph{Weak Systematicity}, \emph{Anonymity}, \emph{Collective Rationality}, and \emph{Unanimity}.

\begin{proof}
Consider the following argumentation framework. An argument $c$ is defeated by two arguments $a$ and $b$ which defeat each others.

\begin{figure}[htbp]
	\begin{center}
  \includegraphics[scale=0.9]{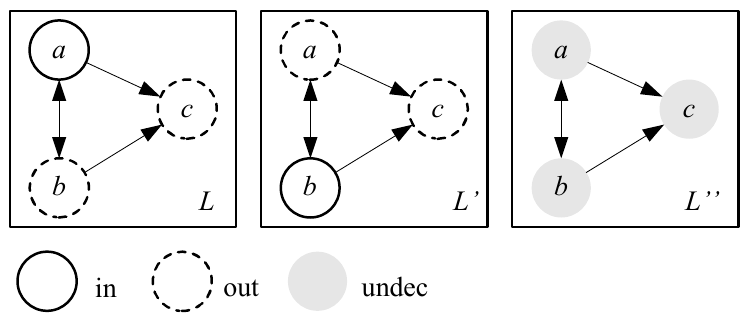}
	\end{center}
  \label{fig:counter_example2}
\end{figure}

Consider the two labellings $L = (\{a\}, \{b, c\}, \{\})$ and $L' = (\{b\}, \{a, c\}, \{\})$. Assume, towards a contradiction, that there exists an aggregation operator $F$ that satisfies \emph{Universal Domain}, \emph{Collective Rationality}, \emph{Weak Systematicity}, \emph{Anonymity} and \emph{Unanimity}. 

By \emph{Universal Domain}, we may consider any profile consisting of legal labellings. Consider the two profiles $\calL = (L, \ldots, L, L', \ldots, L')$ and $\calL' = (L', \ldots, L', L, \ldots, L)$. That is, in $\calL$ half the agents give $L$ and the other half give $L'$, and then in $\calL'$ the agents switch from $L$ to $L'$ and vice versa.

By \emph{Unanimity} we know
\begin{equation}
\label{eq.2.c_out}
[F(\calL)](c) = \myout.
\end{equation}
By \emph{Weak Systematicity} we also know
$[F(\calL)](a) = [F(\calL')](b)$. But since $\calL$ and $\calL'$ are permutations of each other we know $F(\calL) = F(\calL')$ by \emph{Anonymity} and so we obtain
\begin{equation}
\label{eq.2.a_b_same}
[F(\calL)](a) = [F(\calL)](b). 
\end{equation}
But there is no complete labelling simultaneously satisfying (\ref{eq.2.c_out}) and (\ref{eq.2.a_b_same}). Contradiction. Hence no $F$ can exist.

\end{proof}
\end{theorem}

\end{subequations}

One might note that all of the above theorems exploit the use of profiles that include ties. Then, one would ask: What if we relax \emph{Universal Domain} to \emph{No-Tie Universal Domain}? Do we still have impossibility results then? Following, we show that an aggregation operator which satisfies \emph{No-Tie Universal Domain} (but not necessarily \emph{Universal Domain}) cannot also satisfy \emph{Weak Systematicity}, \emph{Anonymity}, \emph{Collective Rationality}, and \emph{Supportiveness} together.

\begin{subequations}
\begin{theorem}\label{thm:impossibility3}
There exists an argumentation framework $\AF$ such that, for any set of agents whose cardinality is divisible by three, there exists no labelling aggregation operator satisfying \emph{No-Tie Universal Domain}, \emph{Weak Systematicity}, \emph{Anonymity}, \emph{Collective Rationality}, and \emph{Supportiveness}.

\begin{proof}
Consider the following argumentation framework. An argument $a$ is defeated by two arguments $b$ and $c$. Argument $b$ (resp. $c$) defeats and is defeated by argument $b'$ (resp. $c'$).

\begin{figure}[htbp]
	\begin{center}
  \includegraphics[scale=0.9]{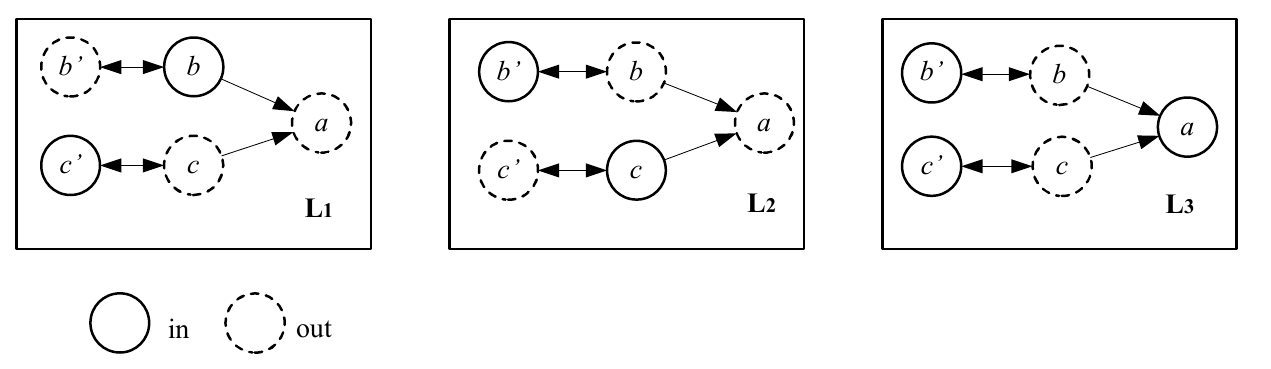}
	\end{center}
  \label{fig:counter_example5}
\end{figure}

Consider the three labellings $L_1 = (\{b, c'\},\{a, b', c\}, \{\})$, $L_2 = (\{b', c\},\{a, b, c'\},$ $\{\})$ and $L_3 = (\{a, b', c'\},\{b, c\}, \{\})$. 

Assume, towards a contradiction, that there exists an aggregation operator $F$ that satisfies \emph{No-Tie Universal Domain}, \emph{Collective Rationality}, \emph{Weak Systematicity}, \emph{Anonymity} and \emph{Supportiveness}. 

By \emph{No-Tie Universal Domain}, we may consider any profile consisting of legal labellings as long as it does not cause a tie. We consider here three agents, but the same proof can be shown for any set of agents that is divisible by three. Consider the three profiles $\calL = (L_1,L_2,L_3)$, $\calL' = (L'_1,L'_2,L'_3) = (L_3,L_1,L_2)$ and $\calL'' = (L''_1,L''_2,L''_3) = (L_2,L_3,L_1)$.

Since $\forall i, L_i(a) = L'_i(c)$, then by \emph{Weak Systematicity} we know:

\begin{equation}
\label{eq.3.ws1}
[F(\calL)](a) = [F(\calL')](c)
\end{equation}

But since $\calL$ and $\calL'$ are permutations of each other we know $[F(\calL)] = [F(\calL')]$ by \emph{Anonymity} and so we obtain
\begin{equation}
\label{eq.3.an1}
[F(\calL)](c) = [F(\calL')](c). 
\end{equation}

From Eq.\ref{eq.3.ws1} and Eq.\ref{eq.3.an1}:

\begin{equation}
\label{eq.3.a_c_same}
[F(\calL)](a) = [F(\calL)](c). 
\end{equation}

Similarly, since $\forall i, L_i(b) = L''_i(c)$, then by \emph{Weak Systematicity} we know:

\begin{equation}
\label{eq.3.ws2}
[F(\calL)](b) = [F(\calL'')](c)
\end{equation}

But since $\calL$ and $\calL''$ are permutations of each other we know $[F(\calL)] = [F(\calL'')]$ by \emph{Anonymity} and so we obtain
\begin{equation}
\label{eq.3.an2}
[F(\calL)](c) = [F(\calL'')](c). 
\end{equation}

From Eq.\ref{eq.3.ws2} and Eq.\ref{eq.3.an2}:

\begin{equation}
\label{eq.3.b_c_same}
[F(\calL)](b) = [F(\calL)](c). 
\end{equation}

From Eq.\ref{eq.3.a_c_same} and Eq.\ref{eq.3.b_c_same}:

\begin{equation}
\label{eq.3.a_b_c_same}
[F(\calL)](a) = [F(\calL)](b) = [F(\calL)](c). 
\end{equation}

The last equation suggests that $a$, $b$, and $c$ have the same collective labelling. However, by \emph{Collective Rationality}, the only legal labelling that satisfy Eq.\ref{eq.3.a_b_c_same} is $\myundec$:
\begin{equation}
\label{eq.3.a_b_c_undec}
[F(\calL)](a) = [F(\calL)](b) = [F(\calL)](c) = \myundec. 
\end{equation}
However, $F$ satisfies \emph{Supportiveness} by assumption. Contradiction.

\end{proof}
\end{theorem}

\end{subequations}
One can draw a connection between this result and the previous one. Relaxing \emph{Universal Domain} to \emph{No-Tie Universal Domain}, introduces another impossibility result, in which \emph{Unanimity} is replaced with the stronger postulate \emph{Supportiveness}. Additionally, one can compare this result to the analogue of Arrow's theorem in judgment aggregation \cite{list2010introduction}, which involves \emph{Unanimity}, \emph{Independence}, and \emph{Non-dictatorship}, the weaker versions of \emph{Supportiveness}, \emph{Weak Systematicity}, and \emph{Anonymity} respectively in our theorem. However, their result also involves \emph{completeness}, i.e. no proposition can be collectively undecided, which we do not have as a condition in our result.

The above impossibility results highlight a major barrier to reaching good collective judgment about argument evaluation in general. These establish the limits of aggregation in the context of argumentation, and come in accordance with the similar topics of aggregation such as preference aggregation \cite{arrow:1951} and judgment aggregation \cite{list:pettit:2002}. Unfortunately, there is no escape from violating the involved conditions or accepting irrational aggregate argument labellings without somewhat lowering our standards in terms of desirable criteria.

\section{Collective Rationality Postulates}\label{section:avoiding}

In this section, we characterise \emph{Collective Rationality} in terms of conditions that need to be satisfied by profiles. To do this, we need to go back to the definition of legal (i.e. \emph{complete}) labelling (Definition \ref{definition:CompLabelling}), and break it down into further constituents defined over the outcome of an aggregation operator.  

The following condition, which we call \emph{IN-Collective Rationality (IN-CR)}, requires that if an argument $a$ is collectively accepted by the agents, then the agents must collectively reject all counter-arguments against $a$.

\begin{quote}{\bf IN-Collective Rationality (IN-CR)}
For any profile $\calL$ and $a \in \calA$, if $[F(\calL)](a)= \myin$ then:
\[\nexists b \in \calA, \text{ s.t. } (b \rightharpoonup a \wedge [F(\calL)](b)= \myin) \tag{\bf IN-CR1}\]
and
\[\nexists b \in \calA, \text{ s.t. } (b \rightharpoonup a \wedge [F(\calL)](b)= \myundec) \tag{\bf IN-CR2}\]
\end{quote}

Note that IN-CR1, the first part of IN-CR, represents the the condition of \emph{conflict-freeness} applied on the output. The condition of \emph{conflict-freeness} is usually agreed on as a minimal reasonable condition in argument evaluation.

We present now the \emph{OUT-Collective Rationality (OUT-CR)} condition. Intuitively, this condition means that if an argument $a$ is collectively rejected by the agents, then the agents must also collectively agree on accepting at least one of the counter-arguments against $a$. 
\begin{quote}{\bf OUT-Collective Rationality (OUT-CR)}
For any profile $\calL$ and $a \in \calA$, if $[F(\calL)](a)= \myout$ then $\exists b \in \calA$, such that $b \rightharpoonup a$ and $[F(\calL)](b)= \myin$.
\end{quote}
We present now the \emph{UNDEC-Collective Rationality (UNDEC-CR)} condition. An argument must be labelled $\myundec$ if and only if: (i) it is not the case that all of its defeaters are $\myout$, that is, at least one of its defeaters is $\myundec$; and (ii) none of its defeaters is $\myin$. 
\begin{quote}{\bf UNDEC-Collective Rationality (UNDEC-CR)}
For any profile $\calL$ and $a \in \calA$, if $[F(\calL)](a)= \myundec$ then:
\[\nexists b \in \calA,\text{ s.t. }(b \rightharpoonup a \wedge[F(\calL)](b)= \myin)\tag{\bf UNDEC-CR1}\]
and 
\[\exists b \in \calA, \text{ s.t. }(b \rightharpoonup a\wedge [F(\calL)](b)= \myundec)\tag{\bf UNDEC-CR2}\]
\end{quote}

The following result follows immediately from the definitions.
\begin{proposition}\label{prop:Condorcet}
An argument aggregation operator $F$ satisfies \emph{Collective Rationality} if and only if for each profile $\calL = (L_1, \ldots, L_n)$ in its domain, it satisfies the \emph{IN-CR}, \emph{OUT-CR}, and \emph{UNDEC-CR} conditions.
\end{proposition}

\section{Plurality Rule with Classical Semantics}\label{section:semantics}

In this section, we analyse the performance of AWPR with respect to \emph{Collective Rationality} when agents labellings are restricted to some classical semantics (i.e. \emph{complete}, \emph{grounded}, \emph{stable}, \emph{semi-stable}, and \emph{preferred}). This investigation gives a novel meaning to classical semantics in social choice settings. Rather than simply being compared by their logical rigour from the perspective of a single agent, semantics are compared based on the extent to which they facilitate collectively rational agreement among agents.

Our strategy will be based on the following approach. Since, by Proposition \ref{prop:Condorcet}, \emph{Collective Rationality} arises iff \emph{IN-CR}, \emph{OUT-CR}, and \emph{UNDEC-CR} are satisfied, it is enough to check whether AWPR satisfies those properties.

\subsection{Complete Semantics}
Since the complete semantics generalises other classical semantics, we provide analysis for it first. Every property that is satisfied by AWPR when individuals' labellings are \emph{complete} labellings would be also satisfied by AWPR when individuals' labellings are restricted to the other classical semantics that we consider. 

It is very interesting to see that, as the proposition below shows, when agents collectively accept an argument, the structure of the AWPR will ensure that they will not collectively accept any of its defeaters:
\begin{subequations}
\begin{proposition}\label{lemma:nocollectivein}
AWPR satisfies \emph{IN-CR1}. Using the argument-wise plurality rule, given any profile $\calL = (L_1, \ldots, L_n)$, if an argument $a$ is collectively accepted, none of its defeaters will be collectively accepted. Formally, if $[M(\calL)](a)= \myin$ for some arbitrary $a \in \calA$ then $\nexists b \in \calA$, such that $b \rightharpoonup a$ and $[M(\calL)](b)= \myin$.
\begin{proof}
Suppose that $[M(\calL)](a)= \myin$ holds. By definition:

\begin{equation} \label{equation.le.1}
|\{i: L_{i}(a) = \myin \}|> |\{i: L_{i}(a)= \myout\}|
\end{equation}
Since each $L_i$ is a legal labelling, an agent who votes $\myin$ for $a$ must also vote $\myout$ for each defeater of $a$. Therefore:
\begin{equation}\label{equation.le.2}
\forall b \rightharpoonup a \ \ \ |\{i: L_{i}(b) = \myout \}| \geq |\{i: L_{i}(a)= \myin\}|
\end{equation}\
We want to show that: $\nexists b \in \calA \ \mbox{such that} \ b \rightharpoonup a \ \mbox{and} \ [M(\calL)](b)= \myin$

Assume (towards contradiction) that the contrary holds. That is, $\exists b^{'} \in \calA$ such that $b^{'} \rightharpoonup a$ and $[M(\calL)](b')= \myin$. Then:

\begin{equation}\label{equation.le.3}
|\{i: L_{i}(b^{'}) = \myin \}|> |\{i: L_{i}(b^{'})= \myout\}| 
\end{equation}


Since every agent who voted $\myin$ for $b^{'}$ would have voted $\myout$ for $a$, we have:

\begin{equation}\label{equation.le.4}
|\{i: L_{i}(a) = \myout \}| \geq |\{i: L_{i}(b^{'})= \myin\}| 
\end{equation}

By Eq.\ref{equation.le.3} and Eq.\ref{equation.le.4}:

\begin{equation}\label{equation.le.5}
|\{i: L_{i}(a) = \myout \}|> |\{i: L_{i}(b^{'})= \myout\}|
\end{equation}

\noindent while from Eq.\ref{equation.le.2} and Eq.\ref{equation.le.5} we have that:

\begin{equation}
|\{i: L_{i}(a) = \myout \}|> |\{i: L_{i}(a)= \myin\}|
\end{equation}

But this contradicts Eq.\ref{equation.le.1} and the assumption that $[M(\calL)](a)= \myin$.
\end{proof}
\end{proposition}
\end{subequations}
It is important to recognise that Proposition \ref{lemma:nocollectivein} is a non-trivial result. It shows that, with AWPR, the postulate IN-CR1 is satisfied. This means, as we mentioned earlier, that AWPR satisfies the ``collective'' version of \emph{conflict-freeness}, a condition that is usually agreed on as a minimal reasonable condition in argument evaluation. This comes ``for free'' as a result of the intrinsic structure of the individual labellings, leading to coordinated votes. Note, however, that the \emph{IN-CR} postulate is not fully satisfied. Although Proposition \ref{lemma:nocollectivein} guarantees that a collectively accepted argument will never have a collectively accepted defeater, it does not guarantee \emph{IN-CR2}, that none of its defeaters will be collectively undecided. This is demonstrated in the following remark.

\begin{remark}\label{rem:comInCR2}
AWPR violates IN-CR2. If an argument is collectively accepted, some of its defeaters might be collectively undecided. 
\begin{proof}
Suppose argument $c$ has two defeaters, $a$ and $b$. Suppose we have $7$ agents, with votes as shown in Figure \ref{fig:counterExampleLemma1}. Clearly, while $c$ is collectively accepted because $[M(\calL)](c)= \myin$, one of its defeaters is not collectively rejected because $[M(\calL)](b)= \myundec$.
\begin{figure}[ht]
  \centering
  \includegraphics[scale=1]{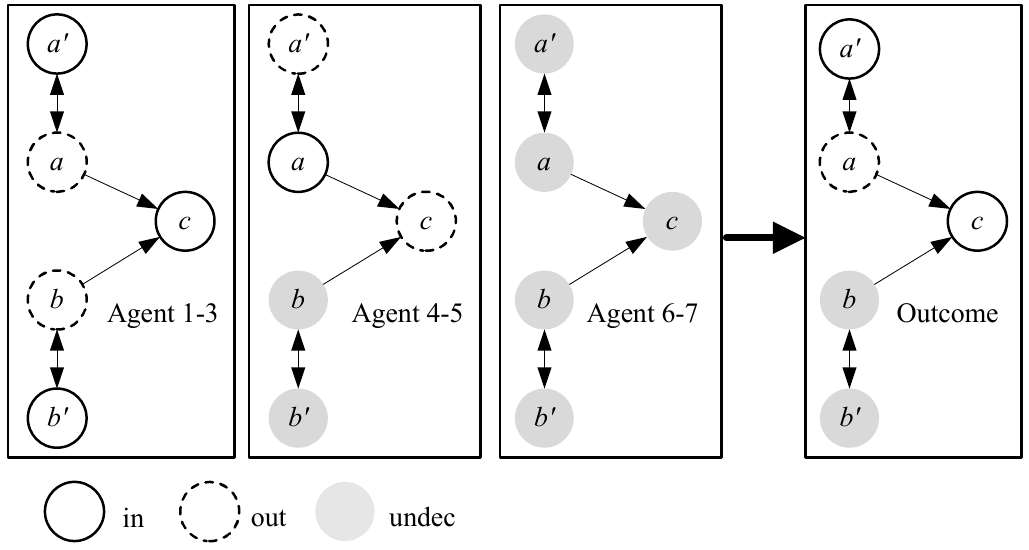}
  \caption{Seven votes collectively accepting $c$, without collectively rejecting $b$}\label{fig:counterExampleLemma1}
\end{figure}
\end{proof}
\end{remark}

As we saw earlier in Example \ref{example:maincounter}, \emph{OUT-CR} is violated by AWPR.

\begin{remark}\label{rem:cdef}
AWPR violates \emph{OUT-CR}. If an argument is collectively rejected, it is not guaranteed that one of its defeaters will be collectively accepted.
\begin{proof}
See Example \ref{example:maincounter} for a counterexample.
\end{proof}
\end{remark}

The following remark shows that there are no intrinsic guarantees for satisfying \emph{UNDEC-CR1}.
\begin{remark}\label{rem:comUndCR1}
AWPR violates \emph{UNDEC-CR1}. If an argument is collectively undecided, it is possible that one of its defeaters will be collectively accepted.
\begin{proof}
Suppose argument $c$ has two defeaters, $a$ and $b$. Suppose we have $7$ agents. Suppose the votes are as shown in Figure \ref{fig:counterExampleUndec2}. We have $[M(\calL)](c)= \myundec$ with $4$ votes, but we have $[M(\calL)](a)= \myin$ with $3$ votes, thus violating the postulate.
\begin{figure}[ht]
  \centering
  \includegraphics[scale=1]{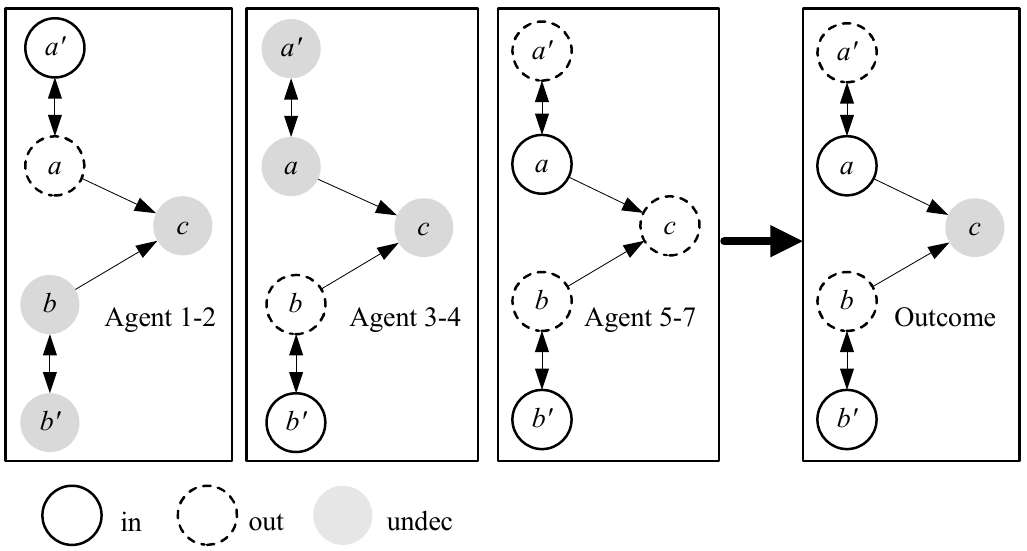}
  \caption{Seven agents collectively undecided on $c$, but collective accepting $a$}\label{fig:counterExampleUndec2}
\end{figure}
\end{proof}
\end{remark}

Similarly, the remark below shows that \emph{UNDEC-CR2} is not intrinsically guaranteed.
\begin{remark}\label{rem:comUndCR2}
AWPR violates \emph{UNDEC-CR2}. If an argument is collectively undecided, it is possible that none of its defeaters will be collectively undecided.
\begin{proof}
Suppose argument $c$ has two defeaters, $a$ and $b$. Suppose we have $3$ agents, with votes as shown in Figure \ref{fig:counterExampleUndec}. Clearly, we have $[M(\calL)](c)= \myundec$, but we have $[M(\calL)](a)= \myout$ and $[M(\calL)](b)= \myout$, which would have required $c$ to be $\myin$.
\begin{figure}[ht]
  \centering
  \includegraphics[scale=1]{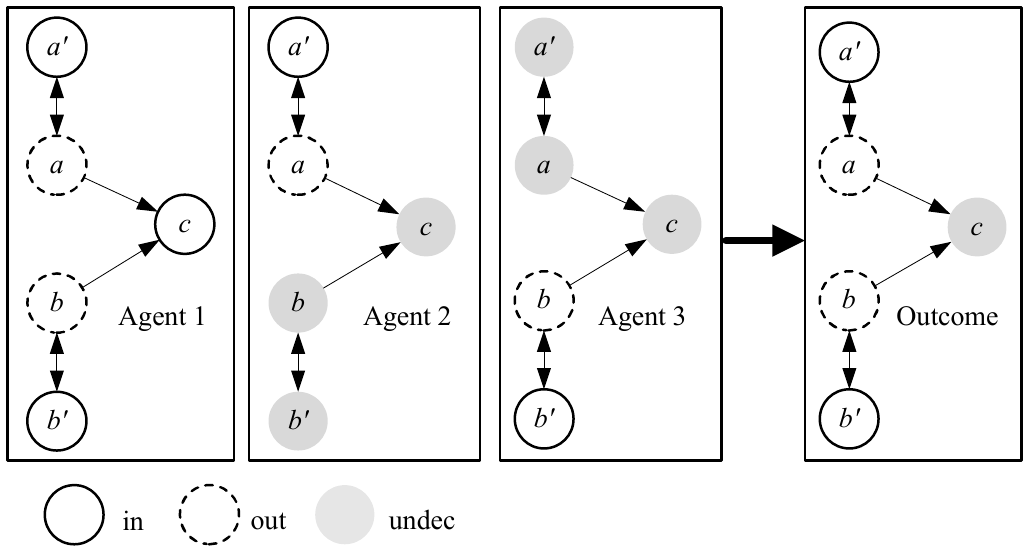}
  \caption{Three votes collectively undecided about $c$, but not collectively undecided about any of its defeaters $a$ or $b$.}\label{fig:counterExampleUndec}
\end{figure}
\end{proof}
\end{remark}

\subsection{Other Classical Semantics}\label{subsection.othersem}
As we noted before, each possible \emph{complete} labelling represents a valid self-defending viewpoint, therefore restricting votes to \emph{complete} labellings is akin to requiring that each vote in judgment aggregation is consistent, or that each preference in preference aggregation is transitive and complete. Other classical semantics are essentially restrictions (i.e. sub-cases) of complete semantics. For example, restricting votes to \emph{preferred} semantics requires each individual to be more committed, maximizing (w.r.t. set-inclusion) the set of accepted (or the set of rejected) arguments, while restricting votes to \emph{semi-stable} semantics requires each individual to be less conservative, minimizing (w.r.t. set-inclusion) the set of arguments about which they are undecided. It is not clear, a priori, what such requirements, applied on the individual, would have on the collective rationality of the outcome of voting.

In this subsection, we provide an analysis for the \emph{grounded}, \emph{stable}, \emph{semi-stable}, and \emph{preferred} semantics as more restricted forms of labellings to choose from. Note that the definition of \emph{Universal Domain}, introduced earlier using \emph{complete} semantics, is now redefined with respect to these semantics, while the definition of \emph{Collective Rationality} is unchanged.  

The following proposition looks trivial but, as we will see, it is the most positive result in this subsection.

\begin{proposition}\label{prop.gro}
If for every argument, agents can only vote for the \emph{grounded} labelling, then $M$ satisfies \emph{IN-CR1}, \emph{IN-CR2}, \emph{OUT-CR}, \emph{UNDEC-CR1} and \emph{UNDEC-CR2}. Equivalently, $M$ satisfies \emph{Collective Rationality}.
\begin{proof}
Trivial since there always exists one grounded labeling \cite{dung:1995,caminada:2006}, and $M$ satisfies \emph{Unanimity}.

\end{proof}
\end{proposition}

%
As a corollary of Proposition \ref{lemma:nocollectivein}, when agents votes are restricted to \emph{stable} (respectively \emph{semi-stable} or \emph{preferred}) labellings, AWPR satisfies IN-CR1.
\begin{corollary}\label{cor.InCR1}
When agents can only vote for \emph{stable} (respectively \emph{semi-stable} or \emph{preferred}) labellings, AWPR satisfies IN-CR1
\begin{proof}
From Proposition \ref{lemma:nocollectivein}, if agents can only vote for \emph{complete} labellings, then AWPR satisfies IN-CR1. Since every \emph{stable} (respectively \emph{semi-stable} or \emph{preferred}) labelling is a complete labelling, then when agents votes are restricted to these semantics, AWPR satisfies IN-CR1. 
\end{proof}
\end{corollary}

\begin{lemma}\label{lem.staInCR2}
When agents can only vote for a \emph{stable} labelling, AWPR satisfies \emph{IN-CR2}. If an argument is collectively accepted, none of its defeaters is collectively undecided.
\begin{proof}
Suppose, towards a contradiction, that there exists an argument that is collectively accepted and one of its defeaters is collectively undecided. Then, by \emph{Supportiveness}, there exists one submitted labelling (by some agent) in which this argument is undecided. However, agents are only allowed to submit a \emph{stable} labelling, and \emph{stable} labellings have no argument labelled undecided. Contradiction.
\end{proof}
\end{lemma}

\begin{remark}\label{rem:staOutCR}
When agents can only vote for \emph{stable} (respectively \emph{semi-stable} or \emph{preferred}) labellings, AWPR violates \emph{OUT-CR}. If an argument is collectively rejected, it is possible that none of its defeaters is collectively accepted.
\begin{proof}
See Example \ref{example:maincounter} for a counterexample. 
\end{proof}
\end{remark}


\begin{lemma}\label{lem.staUndCR1}
When agents can only vote for a \emph{stable} labelling, AWPR satisfies \emph{UNDEC-CR} (i.e. it satisfies both \emph{UNDEC-CR1} and \emph{UNDEC-CR2}). If an argument is collectively undecided, none of its defeaters is collectively accepted, and at least one of its defeaters is collectively undecided.
\begin{proof}
Since in \emph{stable} labelling no argument is labelled undecided, by \emph{Supportiveness}, there is no argument that is collectively undecided. Then, this lemma holds.
\end{proof}
\end{lemma}


%
%

We continue with the \emph{semi-stable} and \emph{preferred} semantics.

\begin{remark}\label{rem:semInCR2}
When agents can only vote for a \emph{semi-stable} (respectively \emph{preferred}) labelling, AWPR violates \emph{IN-CR2}. If an argument is collectively accepted, it is possible that one of its defeaters is collectively undecided.
\begin{proof}
Suppose argument $c_4$ has two defeaters, $a_4$ and $c_6$. Suppose we have $7$ agents, with votes as shown in Figure \ref{fig:SemiUNDCR1}. Clearly, while $c_4$ is collectively accepted because $[M(\calL)](c_4)=\myin$, one of its defeaters, namely $a_4$, is collectively undecided because $[M(\calL)](a_4)=\myundec$.
\end{proof}
\end{remark}

\begin{figure}[ht]
    \centering
  \includegraphics[scale=0.8]{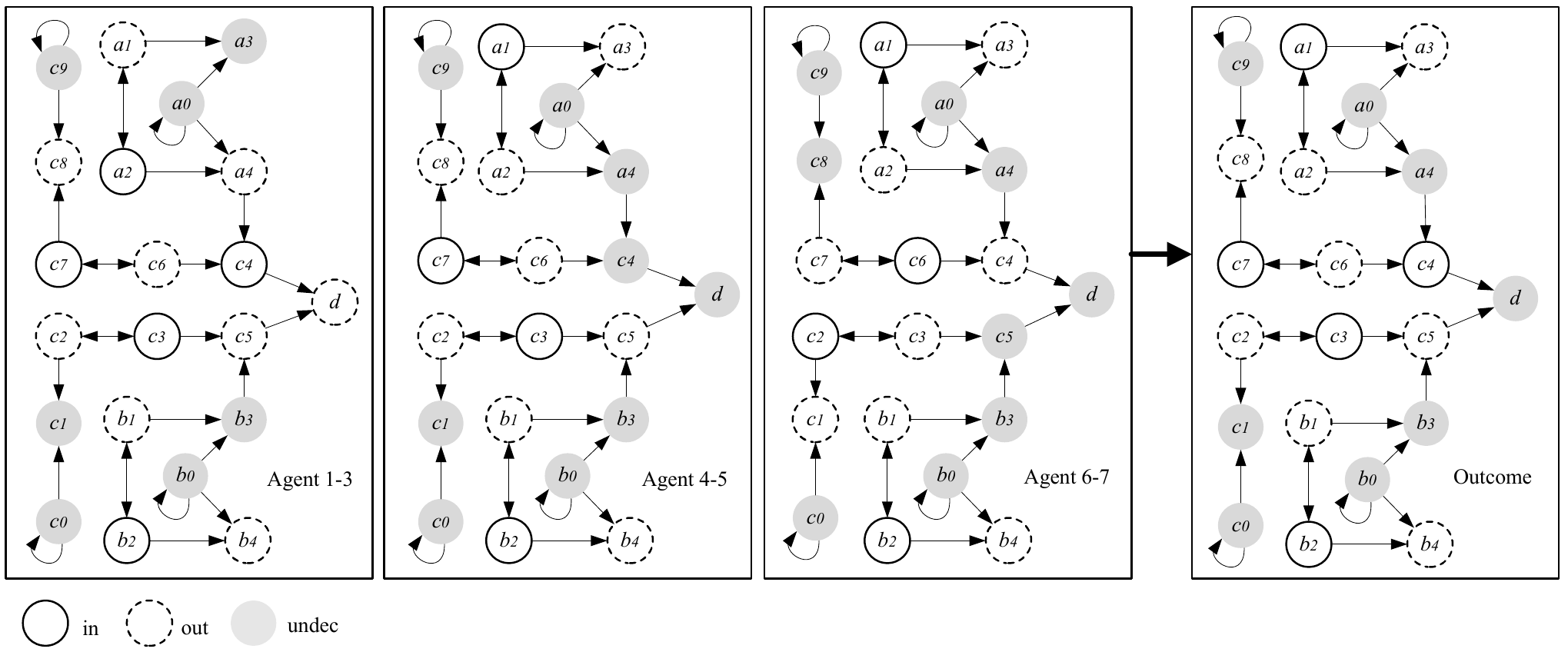}
  \caption{A counterexample shows how, given semi-stable (respectively preferred) semantics, AWPR violates IN-CR2 and UNDEC-CR1.}\label{fig:SemiUNDCR1}
\end{figure}

\begin{remark}\label{rem:semUndCR1}
When agents can only vote for a \emph{semi-stable} (respectively \emph{preferred}) labelling, AWPR violates \emph{UNDEC-CR1}. If an argument is collectively undecided, it is possible that one of its defeaters is collectively accepted.
\begin{proof}
Suppose argument $d$ has two defeaters, $c_4$ and $c_5$. Suppose we have $7$ agents, with votes as shown in Figure \ref{fig:SemiUNDCR1}. Clearly, while $d$ is collectively undecided because $[M(\calL)](d)=\myundec$, one of its defeaters, namely $c_4$, is collectively accepted because $[M(\calL)](c_4)=\myin$.
\end{proof}
\end{remark}

\begin{remark}\label{rem:semUndCR2}
When agents can only vote for a \emph{semi-stable} (respectively \emph{preferred}) labelling, AWPR violates \emph{UNDEC-CR2}. If an argument is collectively undecided, it is possible that none of its defeaters is collectively undecided.

\begin{proof}
Suppose argument $c$ has two defeaters, $a_4$ and $b_3$. Suppose we have $3$ agents, with votes as shown in Figure \ref{fig:SemiWellBeh}. Clearly, while $c$ is collectively undecided because $[M(\calL)](c)= \myundec$, none of its defeaters is collectively undecided.
\end{proof}
\end{remark}

\begin{figure}[ht]
    \centering
  \includegraphics[scale=0.8]{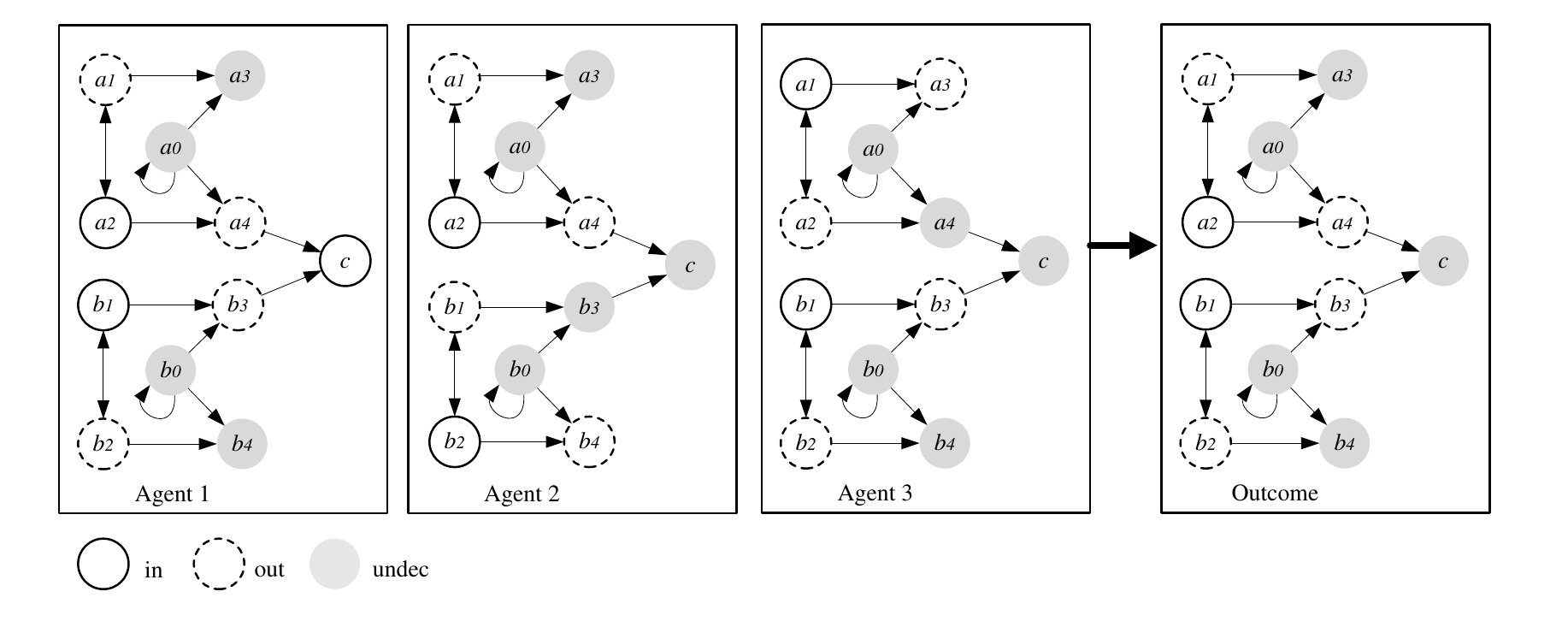}
  \caption{A counterexample shows how, given semi-stable (respectively preferred) semantics, AWPR violates UNDEC-CR2.}\label{fig:SemiWellBeh}
\end{figure}

%
%

To sum up, the only restriction that would satisfy the \emph{Collective Rationality} is the \emph{grounded} semantics (Proposition \ref{prop.gro}). This is trivially true because only one \emph{grounded} labelling exists. However, \emph{stable} semantics violates \emph{Collective Rationality} only because it violates \emph{OUT-CR}. As for the \emph{semi-stable} and \emph{preferred} semantics, they only satisfy \emph{IN-CR1}, a property they inherit from the \emph{complete} semantics. Refer to Table \ref{tab:semSum} for a summary of the results we have found.

\begin{table}[htbp]
  \centering
  \footnotesize
   \begin{tabular}{|c|c|c||c||c|c|}
    \hline
   \multirow{2}{*}{\textbf{Semantics}}  & \multicolumn{2}{c||}{\textbf{IN-CR}} & \multirow{2}{*}{\textbf{OUT-CR}} & \multicolumn{2}{c|}{\textbf{UNDEC-CR}} \\
  
    \cline{2-3}\cline{5-6}
                                        & \textbf{IN-CR1}                      & \textbf{IN-CR2}&                    &                                 \textbf{UND-CR1} & \textbf{UND-CR2} \\
   
    \hline

   \multirow{2}{*}{\textbf{Grounded}} &Yes  & Yes                  &Yes                   &Yes                   &Yes\\
                      &(Prop. \ref{prop.gro})&(Prop. \ref{prop.gro})&(Prop. \ref{prop.gro})&(Prop. \ref{prop.gro})&(Prop. \ref{prop.gro})\\
    \hline
   
   \multirow{2}{*}{\textbf{Stable}}& Yes                     & Yes                     & No                        & Yes & Yes \\
                &(Cor. \ref{cor.InCR1})&(Lem. \ref{lem.staInCR2})&(Rem. \ref{rem:staOutCR})& (Lem. \ref{lem.staUndCR1})&(Lem. \ref{lem.staUndCR1})\\

   \hline
 
   \multirow{2}{*}{\textbf{Semi-stable}}& Yes & No               & No                      & No                       & No \\
                &(Cor. \ref{cor.InCR1})&(Rem. \ref{rem:semInCR2})&(Rem. \ref{rem:staOutCR})&(Rem. \ref{rem:semUndCR1})&(Rem. \ref{rem:semUndCR2})\\
   \hline
   \multirow{2}{*}{\textbf{Preferred}}& Yes & No               & No                      & No                       & No \\
                &(Cor. \ref{cor.InCR1})&(Rem. \ref{rem:semInCR2})&(Rem. \ref{rem:staOutCR})&(Rem. \ref{rem:semUndCR1})&(Rem. \ref{rem:semUndCR2})\\
   \hline
   \multirow{2}{*}{\textbf{Complete}} &Yes & No                      & No                  & No                       & No\\
          &(Prop. \ref{lemma:nocollectivein})&(Rem. \ref{rem:comInCR2})&(Rem. \ref{rem:cdef})&(Rem. \ref{rem:comUndCR1})&(Rem. \ref{rem:comUndCR2})\\
    \hline

    \end{tabular}%
  \caption{The \emph{Collective Rationality} properties that are satisfied/violated by AWPR given different semantics.}\label{tab:semSum}
\end{table}%

  

\section{Restricting the Domain of Argumentation Graphs to Satisfy Collective Rationality}\label{section:graph}
In an earlier section, we showed that, AWPR violates \emph{Universal Domain} and \emph{Collective Rationality}. In this section, we investigate whether AWPR can satisfy \emph{Collective Rationality} by restricting the argumentation framework to graphs with certain graph-theoretical properties. We show that graphs consisting of disconnected issues (a notion we define below) and graphs in which arguments have limited defeaters (in some sense) guarantee collectively rational outcomes when the AWPR is used.

\subsection{Disconnected Issues}
The notion of ``issue'' was defined in \cite{booth2012quantifying} in order to quantify disagreement between graph labellings. In this section, we use this notion to provide a possibility result.

Crucial to the definition of the ``issue'' is the concept of ``in-sync''. Two arguments $a$ and $b$ are said to be \emph{in-sync} if the (\emph{complete}) label of one cannot be changed without causing a change of equal magnitude to the label of the other.

\begin{definition}[in-Sync $\equiv$ \cite{booth2012quantifying}]
Let $\comp{\AF}$ be the set of all complete labellings for argumentation framework $\AF=\langle \calA, \rightharpoonup \rangle$. We say that two arguments $a, b \in \calA$ are in-sync ($a \equiv b$):
\begin{equation}
a \equiv b \text{ iff } (a \equiv_1 b \; \vee \; a \equiv_2 b)
\end{equation}
where:
\begin{itemize}
\item $a \equiv_1 b$ iff $\forall L \in \comp{\AF}:$ $L(a)=L(b)$.
\item $a \equiv_2 b$ iff $\forall L \in \comp{\AF}:$ $(L(a)=\myin$ $\Leftrightarrow$ $L(b)=\myout)$ $\wedge$ $(L(a)=\myout$ $\Leftrightarrow$ $L(b)=\myin)$
\end{itemize}
\end{definition}

This relation forms an equivalence relation over the arguments, and the equivalence classes are called ``issues''. 

\begin{definition}[Issue \cite{booth2012quantifying}]
Given the argumentation framework $\AF=\langle \calA, \rightharpoonup \rangle$, a set of arguments $\calB \subseteq \calA$ is called an issue iff
it forms an equivalence class of the relation \emph{in-Sync ($\equiv$)}.
\end{definition}

For example, in Figure \ref{fig:issueEx}, the graph consists of three issues, namely $\{a_1\}$, $\{a_2,a_3\}$, and $\{a_4,a_5\}$.

\begin{figure}[ht]
    \centering
  \includegraphics[scale=1.2]{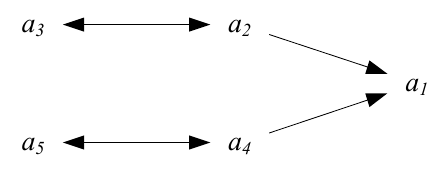}
  \caption{An example about issues.}\label{fig:issueEx}
\end{figure}

The following lemma is crucial in showing the main result of this subsection. We show that if the defeaters of an argument belong to the same issue as the argument, then the collective labelling of this argument chosen by AWPR is always a legal labelling.

\begin{subequations}
\begin{lemma}\label{lem:issue}
Let $\AF=\langle \calA, \rightharpoonup \rangle$ be an argumentation framework. Let $a \in \calA$ be an argument in this framework. If every defeater of $a$ (call it $b$) belongs to the same issue of $a$ (i.e. $\forall b \in a^-$: $b \equiv a$), then AWPR would always produce a legal collective labelling for argument $a$.
\begin{proof}
Let $b_1,\ldots,b_m \in \calA$ such that $b_j \in a^-$ and $a \equiv b_j$ $\forall j=1,\ldots,m$. Then, for every complete labelling $L$:
\begin{equation}\label{eq.le.issue.1.out}
L(a)=\myout \Leftrightarrow L(b_1)=\myin \Leftrightarrow \ldots \Leftrightarrow L(b_m)=\myin
\end{equation} 
\begin{equation}\label{eq.le.issue.1.in}
L(a)=\myin \Leftrightarrow L(b_1)=\myout \Leftrightarrow \ldots \Leftrightarrow L(b_m)=\myout
\end{equation} 
\begin{equation}\label{eq.le.issue.1.undec}
L(a)=\myundec \Leftrightarrow L(b_1)=\myundec \Leftrightarrow \ldots \Leftrightarrow L(b_m)=\myundec
\end{equation} 

From Equations \ref{eq.le.issue.1.out}, \ref{eq.le.issue.1.in}, and \ref{eq.le.issue.1.undec}, for every labelling profile $\calL = (L_1,\ldots,L_n)$:
\begin{equation}\label{eq.le.issue.2.out}
|\{i:L_i(a)=\myout\}| = |\{i:L_i(b_1)=\myin\}| = \ldots = |\{i:L_i(b_m)=\myin\}|
\end{equation} 
\begin{equation}\label{eq.le.issue.2.in}
|\{i:L_i(a)=\myin\}| = |\{i:L_i(b_1)=\myout\}| = \ldots = |\{i:L_i(b_m)=\myout\}|
\end{equation} 
\begin{equation}\label{eq.le.issue.2.undec}
\begin{aligned}[b]
|\{i:L_i(a)=\myundec\}| = |\{i:L_i(b_1)=\myundec\}| = \ldots = |\{i:L_i(b_m)=\myundec\}|
\end{aligned}
\end{equation} 

From Equations \ref{eq.le.issue.2.out}, \ref{eq.le.issue.2.in}, and \ref{eq.le.issue.2.undec}:
\begin{equation}\label{eq.le.issue.3.out}
[M(\calL)](a)=\myout \Leftrightarrow [M(\calL)](b_1)=\myin \Leftrightarrow \ldots \Leftrightarrow [M(\calL)](b_m)=\myin
\end{equation} 
\begin{equation}\label{eq.le.issue.3.in}
[M(\calL)](a)=\myin \Leftrightarrow [M(\calL)](b_1)=\myout \Leftrightarrow \ldots \Leftrightarrow [M(\calL)](b_m)=\myout
\end{equation}
\begin{equation}\label{eq.le.issue.3.undec}
[M(\calL)](a)=\myundec \Leftrightarrow [M(\calL)](b_1)=\myundec \Leftrightarrow \ldots \Leftrightarrow [M(\calL)](b_m)=\myundec
\end{equation}

From Equations \ref{eq.le.issue.3.out}, \ref{eq.le.issue.3.in}, and \ref{eq.le.issue.3.undec}, AWPR satisfies \emph{IN-CR}, \emph{OUT-CR}, and \emph{UNDEC-CR} with respect to $a$ in this case. Then, $a$ is always legally collectively labelled by AWPR if every defeater of it is in the same issue as $a$. 
\end{proof}
\end{lemma}
\end{subequations}

Given the previous lemma, we show that if the argumentation framework consists of a set of disconnected issues, then AWPR satisfies \emph{Collective Rationality} for this framework.

\begin{theorem}\label{thm.issue}
For every $\AF=\langle \calA, \rightharpoonup \rangle$ that consists of a set of disconnected components (i.e. disconnected subgraphs), each of which forms an issue, the argument-wise plurality rule would always produce collectively rational outcomes. 
\begin{proof}
Since $\AF$ consists of a set of disconnected issues, then $\forall a \in \calA$, $a$ has the following property: $\forall b \in \calA$ such that $b \in a^-$ then $b \equiv a$. From Lemma \ref{lem:issue}, $a$ is always legally collectively labelled by AWPR. Then AWPR satisfies \emph{Collective Rationality} for this $\AF$.
\end{proof}
\end{theorem}

This result shows that under argumentation frameworks that consist of disconnected issues, AWPR always satisfies \emph{collective rationality}. Indeed, as long as all arguments in every connected component are ``in-sync'', the labelling of one argument fully specifies the labelling of all those connected to it. Then, one can think of these disconnected components/issues as a set of independent propositions, and voting is done issue-wise. 

\subsection{Limited Defeaters}
Now we move to another condition. It simply states that the defeaters of any argument are limited by the flexibility of labelling of these defeaters. To illustrate the latter term, we use a concept called the ``justification status'', which is defined in \cite{wu2010labelling}. Intuitively, the justification status of an argument is the set of possible labellings that this argument can take.

\begin{definition}[Justification Status \cite{wu2010labelling}]
Let $\AF=\langle \calA, \rightharpoonup \rangle$ be an argumentation framework, and $a \in \calA$ some argument. The justification status of $a$ is the outcome yielded by the function $\calJ\calS: \calA \rightarrow 2^{\{\myin,\myout,\myundec\}}$ such that $\calJ\calS(a) = \{L(a)|L \in \comp{\AF}\}$.
\end{definition} 
There are six possible justification statuses. Neither $\emptyset$ nor $\{\myin,\myout\}$ is a possible justification status. The former is because each argumentation framework has at least one \emph{complete} labelling. The later is because of the following theorem.
\begin{theorem}[{\cite[Theorem 2]{wu2010labelling}}]\label{thm.js.io.u}
Let $\AF=\langle \calA, \rightharpoonup \rangle$ be an argumentation framework, and $a \in \calA$ some argument. If $\AF$ has two complete labellings $L_1$ and $L_2$ such that $L_1(a)=\myin$ and $L_2(a)=\myout$, then there exists a labelling $L_3$ such that $L_3(a)=\myundec$.
\end{theorem}

The following lemma shows that an argument with one of its defeaters belong to the same issue as long as all the other defeaters of this argument have the justification status of $\{\myout\}$.

\begin{subequations}
\begin{lemma}\label{lem.js.out-issue}
Let $\AF=\langle \calA, \rightharpoonup \rangle$ be an argumentation framework, and $a,b \in \calA$ two arguments such that $b \rightharpoonup a$. If the following holds:
\begin{equation}
\forall c \neq b: (c \rightharpoonup a \Rightarrow \calJ\calS(c)=\{\myout\})
\end{equation}
Then $a$ and $b$ belong to the same issue (i.e. $a \equiv b$). Moreover, $a$ is always legally collectively labelled by AWPR.
\begin{proof}
One can show that:
\begin{equation}\label{eq.le.js.out-issue.out}
L(a)=\myout \Leftrightarrow L(b)=\myin
\end{equation} 
\begin{equation}\label{eq.le.js.out-issue.in}
L(a)=\myin \Leftrightarrow L(b)=\myout
\end{equation} 
\begin{equation}\label{eq.le.js.out-issue.undec}
L(a)=\myundec \Leftrightarrow L(b)=\myundec
\end{equation} 
Hence, $a \equiv b$. 

Moreover, in a similar way to Lemma \ref{lem:issue}, one can show that, for every possible profile $\calL = (L_1,\ldots,L_n)$, the following holds:
\begin{itemize}
\item If $[M(\calL)](a)=\myout$ then $[M(\calL)](b)=\myin$ ($b \in a^-$).
\item If $[M(\calL)](a)=\myin$ then $[M(\calL)](b)=\myout$, and by \emph{Unanimity}, $\forall c \neq b: (c \rightharpoonup a \Rightarrow [M(\calL)](c)=\myout)$.
\item If $[M(\calL)](a)=\myundec$ then $[M(\calL)](b)=\myundec$ ($b \in a^-$), and by \emph{Supportiveness}, $\forall c \neq b: (c \rightharpoonup a \Rightarrow [M(\calL)](c) \neq \myin)$.
\end{itemize}
Hence, $a$ is always legally collectively labelled by AWPR.
\end{proof}
\end{lemma}
\end{subequations}

\begin{corollary}\label{cor.js.col.1att}
Let $\AF=\langle \calA, \rightharpoonup \rangle$ be an argumentation framework, and $a \in \calA$ an argument. If $|a^-|=1$ then $a$ is always legally collectively labelled.
\begin{proof}
From Lemma \ref{lem.js.out-issue}, $a$ is always legally collectively labelled by AWPR.
\end{proof}
\end{corollary}

Now we present the main theorem for this subsection. It says if all arguments have limited defeaters then AWPR always produces legally collective labellings. The limitation of the defeaters is characterised in both their number and their justification statuses. 

\begin{theorem}\label{thm.limDef}
Let $\AF=\langle \calA, \rightharpoonup \rangle$ be an argumentation framework. If each argument in $\calA$ has at most one defeater that can be labeled $\myundec$ then AWPR satisfies \emph{Collective Rationality}.
\begin{proof}
Suppose we have an $\AF=\langle \calA, \rightharpoonup \rangle$ s.t. each argument $a \in \calA$ has at most one defeater $b \in a^-$ s.t. $\myundec\in\calJ\calS(b)$. Then, using Theorem \ref{thm.js.io.u}:

\[\forall c \neq b:(c\rightharpoonup a \Rightarrow \calJ\calS(c)=\{\myin\} \vee \calJ\calS(c)=\{\myout\})\]

Now for each argument $a\in\calA$, all defeaters $c$ with $\calJ\calS(c)=\{\myout\}$ have no effect on the label of $a$, so one can remove these defeaters. Additionally, if one of the defeaters $c$ (other than $b$) has $\calJ\calS(c)=\{\myin\}$, then all other defeaters (including $b$) will also have no effect on the label of $a$, so one can also remove those defeaters. As a result, for each argument $a$ we will end up with one of the following:
\begin{itemize}
\item $a$ has only one defeater $b$ and $\myundec\in\calJ\calS(b)$, or
\item $a$ has only one defeater $c$ and $\calJ\calS(c)=\{\myin\}$, or
\item $a$ has no defeaters
\end{itemize}

Note that in the last case, we would have $\calJ\calS(a)=\{\myin\}$, and since AWPR satisfies \emph{Unanimity}, $a$ would be legally collectively labeled $\myin$. As for the first two cases, using Corollary \ref{cor.js.col.1att}, $a$ would be legally collectively labelled.
\end{proof}
\end{theorem}

\subsection{Relating the Two Restrictions}
In this section, we proposed classes of argumentation frameworks that guarantee collective rationality for AWPR. Note that neither of the two classes (given in Theorems \ref{thm.issue} and \ref{thm.limDef}) is a generalisation or a special case of the other. Example \ref{exmp.IssueNotlimDef} shows an $\AF$ that satisfies the condition in Theorem \ref{thm.issue} (i.e. disconnected issues), but violates the condition in Theorem \ref{thm.limDef} (i.e. limited defeaters), while Example \ref{exmp.limDefNotIssue} shows an $\AF$ that satisfies the condition in Theorem \ref{thm.limDef} (i.e. limited defeaters), but violates the condition in Theorem \ref{thm.issue} (i.e. disconnected issues).

\begin{example}\label{exmp.IssueNotlimDef}
Note that the argumentation framework in Figure \ref{fig:IssueNotlimDef} satisfies the condition in Theorem \ref{thm.issue}. All the arguments $a$, $b$, $c$, $d$, and $e$ are in the same issue, so this $\AF$ consists of disconnected issues (only one issue in this case). However, this $\AF$ violates the condition in Theorem \ref{thm.limDef}, since argument $a$ is defeated by two arguments $b$ and $c$, each of these defeaters has a justification status of $\{\myin,\myout,\myundec\}$, and so their justification statuses share $\myundec$. 
\end{example}

\begin{figure}[ht]
    \centering
  \includegraphics[scale=0.8]{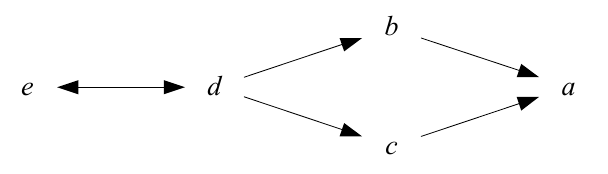}
  \caption{An $\AF$ that satisfies the condition in Theorem \ref{thm.issue}, but violates the condition in Theorem \ref{thm.limDef}.}\label{fig:IssueNotlimDef}
\end{figure}

\begin{example}\label{exmp.limDefNotIssue}
Note that the argumentation framework in Figure \ref{fig:limDefNotIssue} satisfies the condition in Theorem \ref{thm.limDef}. The only argument that is defeated by more than one argument is argument $a$ which has two defeaters $b$ and $c$. Moreover, $\myundec \notin \calJ\calS(c)$, so $\myundec \notin \calJ\calS(b) \cap \calJ\calS(c)$. However, this $\AF$ violates the condition in Theorem \ref{thm.issue}, since it contains two connected issues. The first issue is $\{a,b,d\}$ and the second issue is $\{c,e\}$.
\end{example}

\begin{figure}[ht]
    \centering
  \includegraphics[scale=0.8]{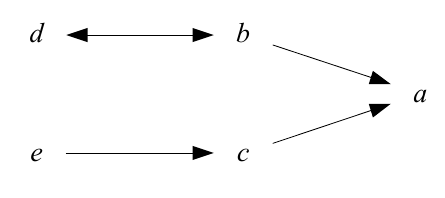}
  \caption{An $\AF$ that satisfies the condition in Theorem \ref{thm.limDef}, but violates the condition in Theorem \ref{thm.issue}.}\label{fig:limDefNotIssue}
\end{figure}

Note that although the two proposed conditions can seem strong, constructing weaker conditions is not an easy task. Consider the graph in Figure \ref{fig:counter_example}. It violates both conditions (it consist of three connected issues $\{a,a'\}$, $\{b,b'\}$ and $\{c\}$, and both defeaters of $c$ i.e. $a$ and $b$ have $\myundec$ in their justification status). However, this framework is very close to frameworks that satisfy one of the two conditions. For example, in the framework in Figure \ref{fig:counter_example}, if we remove the defeat $b \rightharpoonup b'$ we get a graph similar to the one in Figure \ref{fig:limDefNotIssue}, which satisfies the limited defeaters condition. On the other hand, if we remove the defeat $b\rightharpoonup c$ instead, we get a framework consisting of two disconnected issues, namely $\{a,a',c\}$ and $\{b,b'\}$. This suggests the difficulty of finding weaker conditions.

\section{Discussion}\label{section:conclusion}

In this paper, we presented an extensive analysis of social-choice-theoretic aspects of Dung's highly-influential argumentation semantics. 
Argumentation-based semantics have mainly been compared on the basis of how they deal with specific benchmark problems that reflect specific logical structures from the point of view of a single omniscient observer (\eg argument graph structures with odd-cycles, floating defeaters \etc). Recently, it has been argued that argumentation semantics must be evaluated based on more general intuitive principles \cite{baroni:giacomin:2007}. Our work can be seen as a contribution in this direction, focusing on issues relating to multi-agent preferences.

The closest work to the present paper is Caminada and Pigozzi \cite{caminada2011judgment}. In their work, they propose three aggregation operators, namely sceptical, credulous and super credulous. Although the operators satisfy \emph{Collective Rationality}, they violate \emph{Independence}. These operators are also more applicable to scenarios where the compatibility of the collective labelling with each individual's labelling is appreciated or needed. Argument-wise plurality rule, on the other hand, can be applied to classical scenarios where some individuals might naturally disagree with the opinion of the group. Additionally, unlike our work, their work focuses on the proposed operators with only little attention to the general aggregation problem. Only four postulates are proposed, namely \emph{Universal Domain}, \emph{Collective Rationality}, \emph{Anonymity}, and \emph{Independence}, and there are no general impossibility results that holds for \emph{any} operator.

Our results on the aggregation of different argument evaluations by multiple agents provide a new approach for conflict-resolution in multi-agent systems. While this work combines both arguing and voting, two processes that employ different procedures, we assume these two processes are done independently and by different groups of individuals. For example, a jury can vote on the evaluation of arguments that were laid down by the lawyers of two opposing sides. Thus, the arguing part, which happens between the lawyers occurs in an independent step before the voting step, on which our analysis focuses.


Our results contribute to research on aggregation in the context of argumentation. 
The social choice theoretic Arrovian properties have been analyzed in the context of social argument justification in \cite{tohme:etal:2008}. An \emph{extended} argumentation framework $AF^{n}=\langle {\cal A},\ \rightharpoonup_1, \ldots , \rightharpoonup_n \rangle$ is defined, where each $\rightharpoonup_i$, $1\leq i\leq n$, is a particular attack relation among the arguments in ${\cal A}$, representing different attack criteria. Then, the authors define an \emph{aggregate} argumentation framework $AF^{*}=\langle {\cal A},\ {\mathcal F}(\rightharpoonup_1, \ldots , \rightharpoonup_n )\rangle$, where ${\mathcal F}(\rightharpoonup_1, \ldots , \rightharpoonup_n)$ is an attack relation obtained by the aggregation of the individual attack criteria $\rightharpoonup_1, \ldots , \rightharpoonup_n$, via different kinds of mechanisms (e.g. majority voting, qualified voting and mechanisms that can be described by classes of decisive sets). The aggregation of individual attack criteria can not be assimilated to the kind of mechanisms proposed here. In \cite{tohme:etal:2008} an individual may sanction an attack between two given arguments while another individual may not, which in terms of labellings means that for the same pair of arguments there may exist the following two labellings: (\texttt{in}, \texttt{out}) and (\texttt{in}, \texttt{in}). This is impossible in our setting. Hence, the Arrovian properties (e.g. \emph{Collective Rationality}) are conceived differently.

In \cite{bodanza:auday:2009} the authors analyse the problem of aggregating different individual argumentation frameworks over a common set of arguments in order to obtain a unique socially justified set of arguments. One of the procedures considered there is one in which each individual proposes a set of justified arguments and then the aggregation leads to a unique set of socially justified arguments. The AWPR mechanism proposed here fits this procedure for the special case in which individually justified arguments are simply the sets of arguments labelled \texttt{in} for each individual.

There is much work on using an individual agent's preferences to help evaluate arguments (\eg based on given priorities over rules  \cite{prakken:sartor:1997}). But this line of work does not address the preferences of \emph{multiple} agents and how they may be aggregated. In other related work, Bench-Capon \cite{benchcapon:2003} associates arguments with values they promote or demote, and considers different audiences with different preferences over those values. Such preferences determine whether particular defeats among arguments succeed. Thus, one gets different argument graphs, one for each audience. Bench-Capon uses this to distinguish between an argument's \emph{subjective acceptance} with respect to a particular audience, and its \emph{objective acceptance} in case it is acceptable with respect to all possible audiences. Our work differs in two important ways. Firstly, in our framework, an agent (or equivalently, an audience) does not have preferences over individual arguments, but rather preferences over how to evaluate all arguments collectively (\ie over labellings). Secondly, our concern here is not with how individual agents (or audiences) accept an argument, but rather on the possibility of achieving important social-choice properties in the final aggregated labelling.

In relation to aggregation, Coste-Marquis et al explored the problem of aggregating multiple argumentation frameworks \cite{coste-marquis:etal:2007}. Each agent's judgment consists of a different argument graph altogether. This contrasts significantly with our work, in which agents do not dispute the argument graph, but rather how it must be evaluated/labelled. Our setting is more akin to a jury situation, in which all arguments have been presented by the prosecution and defense team, and are visible to the jury members. The jury members themselves do not introduce new arguments, but are tasked with aggregating their individual jugdgments about the arguments presented to them.

Finally, we refer to the work of Rahwan and Larson \cite{rahwan:larson:2008} on strategic behaviour when arguments are distributed among agents, and where these agents may choose to show or hide arguments. Thus, their interest is in how agents contribute to the \emph{construction} of the argument graph itself, which is then evaluated centrally by the mechanism (\eg a judge). In contrast, our work is concerned with how agents individually cast votes on how to evaluate each argument in a \emph{given} fixed graph.

Our work opens new research problems for the computational social choice community. As is the case with other aggregation domains, the aggregation paradox in argument evaluation is an example of a fundamental barrier. Thus the impossibility results are important because they give conclusive answers and focus research in more constructive directions (\eg weakening the desired properties in order to avoid the paradox). 
An algorithmic agenda would complement this research by providing efficient algorithms for such problems. Strategic manipulation, by mis-reporting one's true vote, is also an important area of investigation, especially when such manipulations are exercised by coalitions of agents.

\begin{acknowledgment}
We are grateful to Andrew Clausen for his insightful comments. We are especially grateful to Martin Caminada for discussions surrounding the second impossibility result. Much of this paper was written during a visit of Edmond Awad to the University of Luxembourg which was generously supported by SINTELNET. Richard Booth's work was carried out while working at the University of Luxembourg, and was supported by the FNR
(DYNGBaT project).
\end{acknowledgment}

\bibliographystyle{plain}
\bibliography{mybib}

\end{document}